\documentclass[times, review, 10pt]{elsarticle}
\usepackage{float}
\usepackage{amssymb}
\usepackage{amsmath}
\usepackage{amsthm}
\usepackage{amsfonts}
\usepackage{graphicx}

\usepackage[table]{xcolor}
\usepackage{orcidlink}
\usepackage{array}
\usepackage{algorithm}
\usepackage{algpseudocode}
 \usepackage{booktabs}
 \definecolor{revisionblue}{RGB}{0,70,160}

 \usepackage{multirow} 
\usepackage{hyperref}
\journal{Pattern Recognition}

\begin{document}



\begin{frontmatter}

\author[1]{Qinghui Gong}
\ead{gongqinghui@my.swjtu.edu.cn}

\author[1]{Yihuai Liang\corref{cor1}}
\ead{liangyh@swjtu.edu.cn}

\author[2]{Jinshan Lai}
\ead{202311090910@std.uestc.edu.cn}

\author[3]{Deepak Kumar Jain}

\author[4]{Vitomir  \v{S}truc}
\ead{vitomir.struc@fe.uni-lj.si}

\author[1]{Zhengchun Zhou}
\ead{zzc@swjtu.edu.cn}

\cortext[cor1]{Corresponding author.}

\affiliation[1]{
    organization={School of Information Science and Technology, Southwest Jiaotong University},
    city={Chengdu},
    postcode={611756},
    country={China}
}

\affiliation[2]{
    organization={School of Information and Software Engineering, University of Electronic Science and Technology of China},
    city={Chengdu},
    postcode={610054},
    country={China}
}

\affiliation[3]{
    organization={Key Laboratory of Intelligent Control and Optimization for Industrial Equipment of Ministry of Education, Dalian University of Technology},
    city={Dalian},
    country={China}
}

\affiliation[4]{
    organization={Faculty of Electrical Engineering, University of Ljubljana},
    city={Ljubljana},
    country={Slovenia}
}

\title{GRACE: Adaptive Concept Erasure with Geometry-Guided Retention in Diffusion Models}

%% Abstract
\begin{abstract}
Text-to-image (T2I) diffusion models inevitably internalize sensitive or non-compliant concepts from large-scale pretraining data, necessitating post-hoc concept erasure. However, existing erasure methods often lack explicit constraints on parameter updates, leading to over-intervention and unintended semantic drift. In addition, many methods rely on manually crafted counterfactual supervision, such as surrogate prompts, which incurs substantial data construction costs that limit scalability to new concepts. To address these limitations, we propose GRACE, a structured concept erasure framework designed to enable localized and selective intervention. Specifically, we introduce a semantically weighted sensitive subspace estimation to precisely lock intervention directions, and employ lightweight subspace-constrained adapters to prevent global semantic disturbance. To eliminate the dependency on manual prompt engineering, we design an automatically decoupled safe-anchor mechanism. To mitigate semantic drift induced by excessive intervention, we introduce an energy-driven dynamic gating mechanism that adaptively controls the timing and strength of intervention at inference. Extensive experiments demonstrate that our method achieves a superior balance between erasure effectiveness and generation fidelity. Compared with the average performance of five state-of-the-art (SOTA) concept erasure methods, our method improves the fine-grained NSFW reduction rate by $17.86\%$, while reducing the macro-averaged target CLIP Score and preservation-oriented Fr\'echet Inception Distance (FID) by $4.55\%$ and $52.14\%$, indicating stronger concept suppression with substantially improved preservation of the original model's generative utility.
\end{abstract}

%% Keywords

\end{frontmatter}



\section{Introduction}

Text-to-image (T2I) diffusion models have achieved remarkable progress through large-scale pretraining on web-scale data~\cite{nips/DhariwalN21}. Their flexible prompting mechanisms enable diverse image generation but also facilitate the production of legally or ethically problematic content~\cite{CarliniHNJSTBIW23,schramowski2023safe}. Due to the broad coverage of training data~\cite{schuhmann2022laion}, these models may encode sensitive concepts, including copyrighted characters, artistic styles, and Not-Safe-For-Work (NSFW) content, raising concerns about their safe deployment.

To address these concerns, numerous post-hoc concept erasure methods have been proposed~\cite{iccv/KumariZWS0Z23,iccv/GandikotaMFB23,gandikota2024unified,LiXB0CH25}. However, balancing effective target suppression with non-target semantic preservation remains challenging. As illustrated in Fig.~\ref{fig:figure1}, ESD~\cite{iccv/GandikotaMFB23} suppresses the target concept but introduces substantial changes to scene content and layout, while SPEED~\cite{lispeed} exhibits incomplete erasure and semantic drift in the illustrated example. These limitations highlight the fundamental erasure--preservation trade-off: suppressing target concepts without unnecessarily altering unrelated semantics.

Addressing this trade-off involves two main challenges.
1) Existing optimization based methods~\cite{iccv/GandikotaMFB23,cvpr/LuWLLK24} may induce modifications beyond target-related feature directions, leading to unintended semantic changes. As shown in Fig.~\ref{fig:figure1}, excessive intervention can disrupt scene structure even when the target is suppressed. Additional preservation supervision, such as surrogate concepts or benign reference data, can mitigate this issue but may increase data preparation and optimization costs.
2) Diffusion models generate images through iterative denoising, during which the need for concept intervention varies across timesteps. Applying excessive or poorly timed corrections may alter the generation trajectory and introduce semantic drift. Therefore, effective erasure requires controlling both the scope and timing of intervention.

To address these challenges, we propose GRACE, a geometry-guided concept erasure framework that balances target suppression and semantic preservation.
1) We estimate semantically weighted sensitive subspaces and train lightweight adapters with subspace constraints to concentrate modifications along target-related directions.
2) We construct prompt-specific safe anchors through orthogonal decomposition in CLIP space~\cite{radford2021learning}, providing preservation supervision without manually designed surrogate prompts or external retention datasets.
3) We introduce energy-driven dynamic gating to adaptively regulate adapter activation and intervention strength across denoising timesteps, reducing unnecessary semantic disturbance.

Extensive experiments demonstrate that GRACE achieves a favorable erasure preservation trade-off. Compared with the average performance of five SOTA methods, GRACE improves the fine-grained NSFW reduction rate by 17.86\%, while reducing the macro-averaged target CLIP Score and preservation-oriented FID by 4.55\% and 52.14\%. Cross-model evaluations further examine its applicability across different diffusion architectures. Our main contributions are summarized as follows:

\begin{figure}[t]
  \centering
  \includegraphics[width=0.95\linewidth]{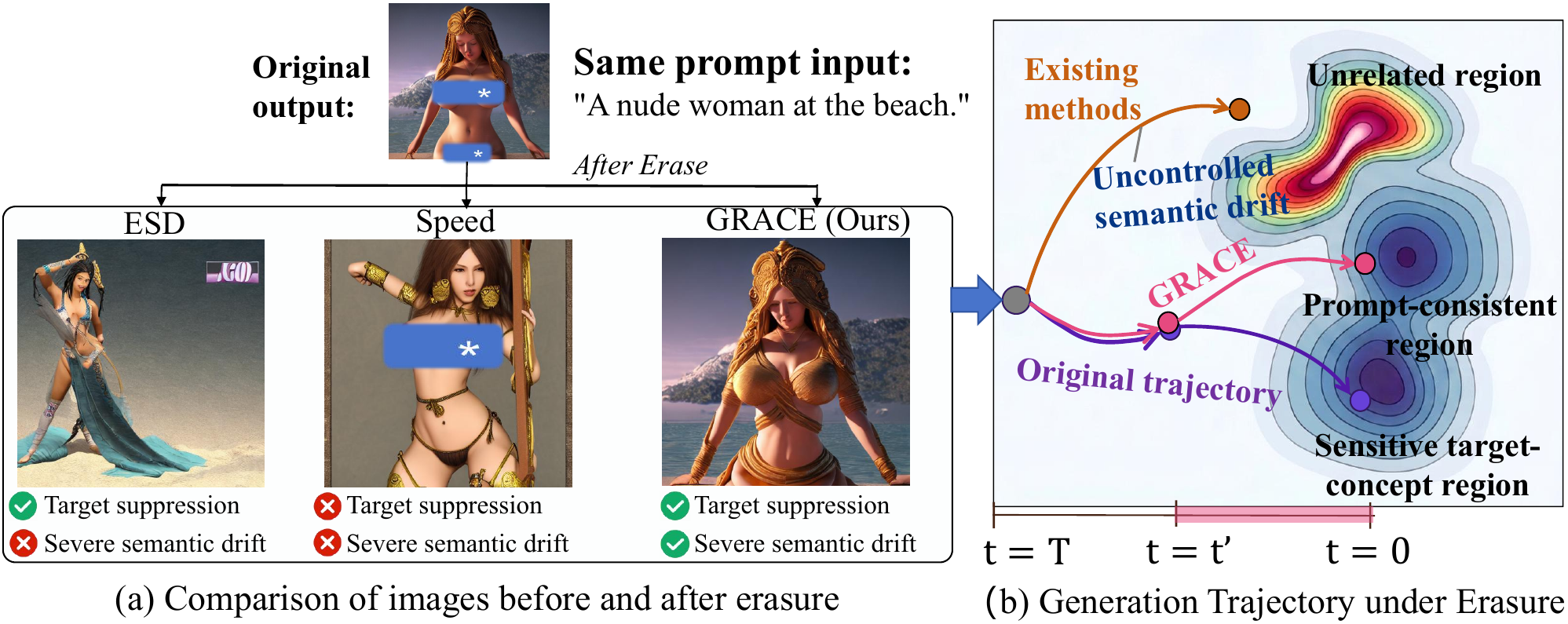}
  \caption{
  Motivation for controlled concept erasure.
  (a) Existing methods may exhibit incomplete erasure or semantic drift, whereas GRACE better preserves scene content and layout while suppressing the target.
  (b) Conceptual generation trajectories illustrating how excessive intervention can deviate from prompt-consistent semantics, while GRACE aims to limit such deviation.
  }
  \label{fig:figure1}
\end{figure}
\begin{itemize}

    \item We propose a novel concept erasure framework that balances target concept erasure with the preservation of the model's original generative capability.

    \item We develop a geometry-guided intervention strategy that combines semantically weighted sensitive subspace estimation with subspace constrained adapters to confine modifications to target-related feature directions. We further devise an energy-driven dynamic gating mechanism to control intervention timing and strength during denoising, reducing deviation from the original generation trajectory.

    \item We design a decoupled safe anchor deriving prompt-specific preservation semantics via CLIP-space orthogonal decomposition, reducing reliance on manual surrogate prompts and external datasets.

\end{itemize}

\section{Related Work}

\textbf{Concept Erasure.}
Concept erasure aims to suppress undesirable concepts in text-to-image diffusion models while preserving their ability to generate unrelated content. Early optimization-based approaches, such as ESD~\cite{iccv/GandikotaMFB23}, MACE~\cite{cvpr/LuWLLK24}, and Co-Erasing~\cite{LiXB0CH25}, achieve erasure through targeted optimization objectives. Alternatively, weight-editing methods, including UCE~\cite{gandikota2024unified}, SPEED~\cite{lispeed}, and CURE~\cite{biswas2025cure}, seek more controlled parameter modifications, while SAFREE~\cite{conf/iclr/YoonYPYB25} provides training-free inference-time intervention. Recent studies have increasingly emphasized preservation and scalability. NLCE~\cite{Shi2026NLCE} incorporates neighboring-concept awareness to reduce collateral suppression, whereas ETC~\cite{Seo2026ETC} models concept distributions for large-scale erasure. Selective subspace unlearning~\cite{Dubey2026SSU} further addresses interference between forgetting and retention objectives through parameter selection and conflict-aware optimization. Despite these advances, achieving reliable target suppression without disturbing related or unrelated semantics remains a central challenge.

\textbf{Subspace-Constrained Techniques.}
Geometric constraints offer a means of controlling model modifications by exploiting the structure of learned representations. AlphaEdit~\cite{iclr/FangJWMSW0C25} employs null-space projection for knowledge editing, while CURE~\cite{biswas2025cure} introduces orthogonal representation editing for concept removal. Recent methods extend this geometric perspective to diffusion models. GLoCE~\cite{Lee_2025_CVPR} combines low-rank representations with conditional gating for localized erasure, while GenErase~\cite{Vardhana2026GenErase} uses semantic subspace protection and geometric gating to enable selective inference-time intervention. Closed-form double projections~\cite{Zhang2026DoubleProjections} and orthogonal concept erasure~\cite{Sun2026OCE} further explore geometrically constrained parameter editing. Beyond spatial and representational selectivity, SDErasure~\cite{miaosderasure} identifies concept-critical denoising timesteps to reduce intervention. Related studies on semantic-aware diffusion editing~\cite{Chen2026FastEdit} and disentangled visual representations~\cite{Huang2026CausalDiffuser} highlight the importance of controlling semantic changes during generation. Building on these insights, GRACE jointly considers feature-space localization, semantic preservation, and inference-time selectivity.

\begin{figure*}[t]
\centering
\includegraphics[width=1.0\textwidth]{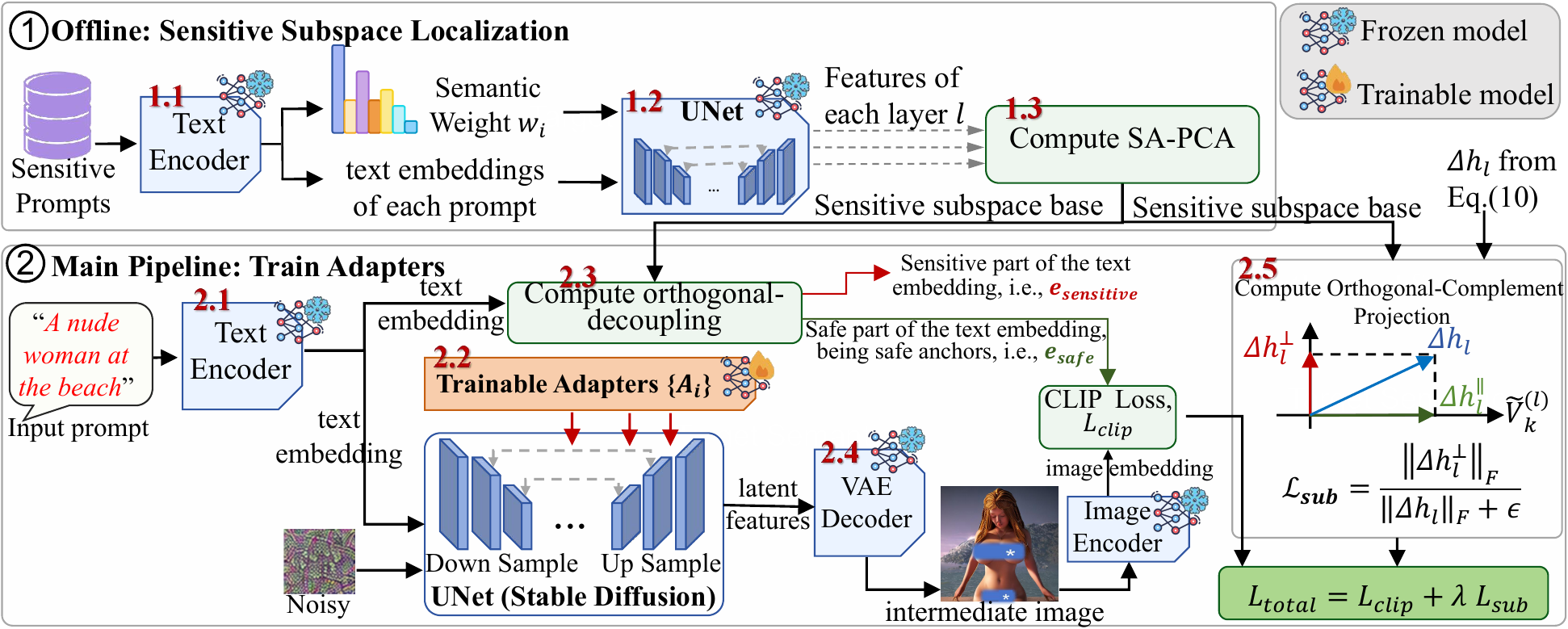} 
\caption{GRACE pipeline for sensitive-subspace estimation and adapter training. SA-PCA estimates concept-specific bases; prompt-derived safe-anchor supervision and off-subspace regularization train residual adapters while pretrained components remain frozen.
}
\label{fig:overflow}
\end{figure*}

\section{Method}

\textbf{Overview.}
GRACE consists of three stages: sensitive-subspace localization,
adapter training, and energy-driven gated inference.
The first two stages are illustrated in Fig.~\ref{fig:overflow}.
In the offline stage,
\textbf{(1.1)} sensitive prompts are encoded by the frozen text encoder
to obtain text embeddings and semantic relevance weights $\omega_i$.
\textbf{(1.2)} The prompts are then processed by the frozen UNet
to extract intermediate features from selected layers.
\textbf{(1.3)} SA-PCA uses these features and semantic weights
to estimate layer-wise sensitive subspace bases $V_K^{(l)}$,
which guide subsequent adapter training and inference.
In the training stage,
\textbf{(2.1)} the input prompt is encoded by the frozen text encoder.
\textbf{(2.2)} Its embedding and a noisy latent are processed by
the frozen UNet equipped with trainable residual adapters $\{A_l\}$.
\textbf{(2.3)} Orthogonal decomposition separates the prompt embedding
into target-aligned and residual components, with the latter
serving as a safe anchor.
\textbf{(2.4)} The predicted clean latent is decoded by the frozen VAE
and encoded into CLIP image space, where the safe-anchor loss
provides semantic preservation supervision.
\textbf{(2.5)} Adapter-induced feature changes are projected onto
the orthogonal complement of the estimated sensitive subspaces.
The resulting off-subspace regularizer and CLIP loss jointly
optimize the adapters while the model remains frozen.
In the inference stage, the sensitive subspaces are reused
to measure projection energy and compute layer-wise activation gates.
Together with timestep-dependent scaling, these gates regulate
adapter contributions during denoising.
For multi-concept erasure, overlap-aware weights further adjust
the composition of concept-specific adapter outputs. The complete pseudocode is provided in \ref{code}.

\subsection{Semantic-Aware Sensitive Subspace Localization}
\label{sec:subspace_localization}

The offline stage aims to localize the target concept within the high-dimensional intermediate representations of the diffusion model. Let $s$ denote the target concept to be erased. Our key observation is that target-related variations can be captured by a compact set of feature directions, forming a layer-wise low-dimensional sensitive subspace. This compact structure is further supported by the empirical layer-wise rank analysis in \ref{sec:appendix_low_rank}. Once identified, the sensitive subspace provides a shared geometric reference for constraining model modifications during training and measuring target-related responses during inference.

To estimate this subspace, we introduce SA-PCA, which incorporates the semantic relevance of each prompt into covariance estimation, thereby emphasizing feature variations more strongly associated with $s$.
Specifically, we construct a target-related prompt set $\{p_i\}_{i=1}^{N}$, where each $p_i$ expresses the target concept $s$ in a different textual context. Each prompt is passed through the frozen diffusion model $\mathcal{M}$ to collect intermediate UNet features. At a selected layer $l$, the feature map has $C_l$ channels and spatial size $H_l \times W_l$. We flatten its spatial dimensions and represent it as
\begin{equation}
\mathcal{H}_i^{(l)}
=
\left\{
h_{i,j}^{(l)} \in \mathbb{R}^{C_l}
\right\}_{j=1}^{M_l},
\qquad
M_l = H_l W_l,
\label{eq:features}
\end{equation}
where $h_{i,j}^{(l)}$ is the $C_l$-dimensional feature vector at the $j$-th flattened location.

Because different prompts express the target concept with different semantic strengths, we assign each prompt a semantic relevance weight. Let $e_s,c_i\in\mathbb{R}^{D}$ be the $\ell_2$-normalized CLIP text embeddings of the target concept $s$ and prompt $p_i$. We define
$
\omega_i=
\max\left(0,c_i^\top e_s\right),
$
where $c_i^\top e_s$ is their cosine similarity. A larger $\omega_i$ indicates that $p_i$ is more strongly aligned with $s$ and contributes more to the estimation of the sensitive subspace.

Using these prompt-level semantic weights, we aggregate the intermediate features across prompts and spatial positions. The weighted mean at layer $l$:
\begin{equation}
\mu^{(l)}=
\frac{1}{Z_l}
\sum_{i=1}^{N}
\omega_i
\sum_{j=1}^{M_l}
h_{i,j}^{(l)},
\qquad
Z_l=
M_l
\sum_{i=1}^{N}
\omega_i,
\label{eq:weighted_mean}
\end{equation}
and the corresponding semantically weighted covariance matrix is
\begin{equation}
\Sigma^{(l)}=
\frac{1}{Z_l}
\sum_{i=1}^{N}
\omega_i
\sum_{j=1}^{M_l}
\left(
h_{i,j}^{(l)}-\mu^{(l)}
\right)
\left(
h_{i,j}^{(l)}-\mu^{(l)}
\right)^\top.
\label{eq:cov}
\end{equation}
This weighting gives greater influence to feature variations produced by prompts that are more strongly associated with the target concept.

We then retain the top $K$ eigenvectors of $\Sigma^{(l)}$:
\begin{equation}
V_K^{(l)}=
\left[
v_1^{(l)},
v_2^{(l)},
\ldots,
v_K^{(l)}
\right]
\in
\mathbb{R}^{C_l\times K},
\label{eq:sensitive_basis}
\end{equation}
where $K$ denotes the retained subspace rank. The columns of $V_K^{(l)}$ form an orthonormal basis of the estimated sensitive subspace, capturing the dominant feature variations that are semantically associated with $s$.

The resulting subspace bank $\{V_K^{(l)}\}_{l\in\mathcal{L}}$ is reused throughout GRACE. The training stage uses this geometric localization to constrain adapter optimization while determining which non-target semantics should be preserved.

% =========================================================
\subsection{Subspace-Constrained Adapter Optimization}
\label{sec:adapter_optimization}

Having localized the target-related feature directions,
we introduce a lightweight residual adapter $A_l$
at each selected layer $l\in\mathcal{L}$.
Each adapter consists of two $1\times1$ convolutional
layers with a ReLU activation:
\begin{equation}
    A_l(h_l)=
\mathrm{Conv}_{1\times1}^{(2)}
\left(
\mathrm{ReLU}\left(
\mathrm{Conv}_{1\times1}^{(1)}(h_l)
\right)
\right).
\end{equation}
The adapter output is added to the original features
through a residual connection. During training,
the pretrained text encoder, UNet backbone, and VAE
decoder remain frozen, while only the adapter
parameters are optimized.

This stage addresses two complementary objectives. First, the optimization requires a semantic target that specifies what should be removed and what prompt semantics should be preserved after erasure. Second, the adapter-induced feature changes should be constrained to the target-related subspace to avoid unnecessary modification of unrelated directions.

\subsubsection{Automatically Decoupled Safe Anchor for Joint Erasure and Retention}
\label{sec:safe_anchor}

Existing methods often rely on auxiliary retain data or manually constructed surrogate prompts to specify preservation semantics. Instead, we introduce an automatically decoupled safe anchor that derives preservation supervision directly from the current training prompt.
Let $E_{\mathrm{text}}$ denote the CLIP text encoder. For the training prompt $c$ and the target concept $s$, we obtain their normalized CLIP text embeddings as
\begin{equation}
e_c
=
\mathrm{Norm}\!\left(E_{\mathrm{text}}(c)\right),
\qquad
e_s
=
\mathrm{Norm}\!\left(E_{\mathrm{text}}(s)\right), 
\label{eq:text_clip}
\end{equation}
where $e_c,e_s\in\mathbb{R}^{D}$.

We then remove the component of $e_c$ aligned with the target direction:
\begin{equation}
\widetilde{e}_{\mathrm{safe}}
=
e_c
-
\left(e_c^\top e_s\right)e_s,
\qquad
e_{\mathrm{safe}}
=
\frac{
\widetilde{e}_{\mathrm{safe}}
}{
\left\|
\widetilde{e}_{\mathrm{safe}}
\right\|_2
}.
\label{eq:safe_anchor}
\end{equation}
where $\left(e_c^\top e_s\right)e_s$ is the projection of the prompt representation onto the target-concept direction. Since $\|e_s\|_2=1$, the residual satisfies $e_s^\top\widetilde{e}_{\mathrm{safe}}=0,$
so $e_{\mathrm{safe}}$ represents the residual prompt semantics after removing the target-aligned component. For numerical stability, the implementation adds a small $\epsilon$ to the normalization denominator.

During training, the predicted clean latent is decoded by the frozen VAE decoder and mapped into the CLIP image space:
\begin{equation}
\widehat{x}_0
=
\mathcal{D}_{\mathrm{VAE}}\!\left(\widehat{z}_0\right),
\qquad
v_{\mathrm{pred}}
=
\mathrm{Norm}\!\left(
E_{\mathrm{img}}\!\left(\widehat{x}_0\right)
\right),
\label{eq:decode_clip}
\end{equation}
where $\mathcal{D}_{\mathrm{VAE}}$ denotes the frozen VAE decoder and $E_{\mathrm{img}}$ is the CLIP image encoder. We then define
\begin{equation}
\mathcal{L}_{\mathrm{CLIP}}
=
1-
v_{\mathrm{pred}}^\top e_{\mathrm{safe}}.
\label{eq:clip}
\end{equation}
Minimizing $\mathcal{L}_{\mathrm{CLIP}}$ aligns the decoded clean-image estimate with the non-target
semantics represented by the safe anchor. By providing semantic supervision at the image level rather than prescribing a reference denoising trajectory at every timestep, this objective allows the model greater flexibility in realizing the desired semantic outcome.

\subsubsection{Subspace-Constrained Residual Regularization}
\label{sec:subspace_regularization}

The safe anchor specifies which semantics should be preserved, but does not constrain how the adapter modifies intermediate representations. We therefore reuse the sensitive subspace from Sec.~\ref{sec:subspace_localization} to restrict adapter-induced feature deviations. Let
\begin{equation}
\Delta h_l
=
h_l^{\mathrm{edit}}
-
h_l^{\mathrm{orig}}
\label{eq:feature_residual}
\end{equation}
denote the adapter-induced feature change at layer $l$, where $h_l^{\mathrm{edit}}$ and $h_l^{\mathrm{orig}}$ are produced by the adapter-augmented model and the frozen original model, under the same noisy latent, prompt condition, and timestep.

Using the sensitive basis $V_K^{(l)}$ obtained in the first stage, we decompose $\Delta h_l$ into its in-subspace and off-subspace components:
\begin{equation}
\Delta h_l^{\mathrm{proj}}
=
V_K^{(l)}
V_K^{(l)\top}
\Delta h_l,
\qquad
\Delta h_l^\perp
=
\Delta h_l
-
\Delta h_l^{\mathrm{proj}}.
\label{eq:residual_decomposition}
\end{equation}
where $\Delta h_l^{\mathrm{proj}}$ represents the modification along target-related directions, while $\Delta h_l^\perp$ represents the component outside the sensitive subspace.

Rather than penalizing the entire adapter-induced change, which would also limit the capacity required for erasure, we penalize only the off-subspace component:
\begin{equation}
\mathcal{L}_{\mathrm{sub}}
=
\frac{1}{|\mathcal{L}|}
\sum_{l\in\mathcal{L}}
\frac{
\left\|
\Delta h_l^\perp
\right\|_F
}{
\left\|
\Delta h_l
\right\|_F+\epsilon
}.
\label{eq:subspace_loss}
\end{equation}
The normalization by $\|\Delta h_l\|_F$ makes the regularizer depend on the relative proportion of off-subspace modification rather than the absolute magnitude of the adapter response.

\subsubsection{Timestep-Aware Optimization}
\label{sec:timestep_optimization}

The safe-anchor loss specifies the semantics to preserve, while the subspace loss restricts where the adapter can modify the representation. During training, we further sample timesteps from an intermediate-noise range:
\begin{equation}
t
\sim
\mathcal{U}
\left\{
T_{\mathrm{min}},
\ldots,
T_{\mathrm{max}}
\right\}.
\label{eq:timestep_sampling}
\end{equation}
where $t$ is the training timestep. Under the original 1,000-step diffusion schedule, we set
$T_{\mathrm{min}}=400$ and $T_{\mathrm{max}}=800$.
This range avoids extremely noisy states with weak semantic information and very late states with limited room for semantic modification.

The complete adapter-training objective is
\begin{equation}
\mathcal{L}_{\mathrm{total}}
=
\mathcal{L}_{\mathrm{CLIP}}
+
\lambda\mathcal{L}_{\mathrm{sub}}.
\label{eq:total_loss}
\end{equation}
where $\mathcal{L}_{\mathrm{CLIP}}$ provides the safe-anchor supervision, $\mathcal{L}_{\mathrm{sub}}$ penalizes off-subspace modification, and $\lambda$ controls their trade-off. We use concept-specific $\lambda$ values and report them in the experimental settings.

\subsection{Energy-Driven Dynamic Gating}
\label{sec:gating}

After training, the adapter does not need to remain equally active throughout generation. At inference time, we reuse the sensitive subspace to determine whether the target concept is currently active and how strongly the adapter should intervene.

\subsubsection{Projection-Energy Layer-wise Gating}

For each selected layer $l$, we average its current feature map over the spatial dimensions and obtain a $C_l$-dimensional feature vector $\bar{h}_l$. We measure its response within the sensitive subspace as
\begin{equation}
E_l
=
\frac{
\left\|
V_K^{(l)\top}\bar{h}_l
\right\|_2
}{
\left\|
\bar{h}_l
\right\|_2+\epsilon
}.
\label{eq:energy}
\end{equation}
where $V_K^{(l)}\in\mathbb{R}^{C_l\times K}$ is the sensitive basis estimated in Sec.~\ref{sec:subspace_localization}, $C_l$ is the channel dimension, $K$ is the sensitive-subspace rank, and $\epsilon$ is a small constant for numerical stability. A larger $E_l$ indicates that the current feature contains a stronger response along the target-sensitive directions.
We convert $E_l$ into a soft activation gate:
\begin{equation}
\gamma_l
=
\sigma
\left(
\alpha(E_l-\tau_l)
\right).
\label{eq:spatial_gate}
\end{equation}
where $\gamma_l\in(0,1)$ controls the activation of adapter $A_l$, $\sigma(\cdot)$ is the sigmoid function, $\tau_l$ is the activation threshold, and $\alpha$ controls the transition sharpness.We set $\alpha=40$; its role in controlling
the sigmoid transition sharpness is discussed
in \ref{sec:appendix_hparams}. When $E_l<\tau_l$, the adapter is weakly activated; when $E_l>\tau_l$, its contribution increases.
The threshold is defined as
\begin{equation}
\tau_l
=
\sqrt{
\frac{K}{C_l}
}
+
\delta.
\label{eq:tau}
\end{equation}
% where $\sqrt{K/C_l}$ provides an analytical reference scale for the projection magnitude of a concept-agnostic feature onto a random $K$-dimensional subspace under the isotropic assumption (see Supplementary Sec.~S2.1), and $\delta$ is an additional margin. We set $\delta=0.03$ for all layers and concepts. This threshold therefore provides a data-free reference for distinguishing target-related responses from incidental projection.
where $\sqrt{K/C_l}$ provides an analytical reference scale for concept-agnostic projection under the random-subspace assumption (see \ref{appendix:threshold_proof}), and $\delta$ is an additional margin. We set $\delta=0.03$ for all layers and concepts. This value provides a small separation from incidental projection while avoiding overly conservative gating.

\subsubsection{Temporal Intervention Scaling.}

The layer-wise gate determines whether the target-related response is strong, but the required intervention strength also depends on the current denoising stage. We therefore introduce a temporal scaling factor:
\begin{equation}
\eta(t)
=
\frac{1}{
1+
\exp\left(
\beta(t-t_{\mathrm{start}})
\right)
}.
\label{eq:temporal_gating}
\end{equation}
where $t$ is the current inference timestep, $t_{\mathrm{start}}$ is the concept-specific transition point, and $\beta$ controls the transition sharpness. We set $\beta=1.0$, the corresponding
temporal transition width is derived
analytically in ~\ref{sec:appendix_hparams}.

The layer-wise gate and temporal scaling jointly control the trained adapter:
\begin{equation}
h_l^{\mathrm{out}}
=
h_l^{\mathrm{in}}
+
\eta(t)\gamma_l
A_l\!\left(h_l^{\mathrm{in}}\right).
\label{eq:output}
\end{equation}
where $h_l^{\mathrm{in}}$ and $h_l^{\mathrm{out}}$ are the input and output features of layer $l$, and $A_l(\cdot)$ is the trained concept-erasing adapter. The factor $\gamma_l$ determines \emph{whether} the adapter should be activated, while $\eta(t)$ determines \emph{how strongly} it should intervene at the current timestep.

\subsection{Overlap-Aware Multi-Concept Composition}
\label{sec:multi_concept}

GRACE naturally extends to multi-concept erasure by composing
concept-specific adapters under independent energy gates. However, when
multiple target concepts share similar sensitive directions, directly summing
their adapter residuals can lead to redundant modifications along overlapping
feature subspaces. To characterize such redundancy, for two target concepts
$m$ and $n$ at layer $l$, we define their normalized subspace overlap as
\begin{equation}
O_{mn}^{(l)}
=
\frac{
\left\|
V_{m}^{(l)\top}V_{n}^{(l)}
\right\|_F^2
}{
\min(K_m,K_n)
},
\end{equation}
where $V_m^{(l)}$ and $V_n^{(l)}$ denote their sensitive-subspace bases.
Since bases are orthonormal, $O_{mn}^{(l)}\in[0,1]$, where a value
close to zero indicates independent sensitive directions and a larger
value indicates stronger geometric overlap.

To account for whether each concept is active at the current layer
and denoising stage, we define its intervention activity as
$a_{m,l}(t)=\eta_m(t)\gamma_{m,l}.$
We then compute the activity-weighted overlap of concept $m$ with the
target concepts:
\begin{equation}
\bar O_{m,l}(t)
=
\frac{
\sum_{n\neq m}
O_{mn}^{(l)}a_{n,l}(t)
}{
\sum_{n\neq m}a_{n,l}(t)+\epsilon
},
\end{equation}
and use it to define an overlap-aware attenuation factor
\begin{equation}
w_{m,l}(t)
=
\frac{1}{
1+\bar O_{m,l}(t)
}.
\end{equation}
During multi-concept inference, $w_{m,l}(t)$ reweights the gated residual
of each concept before
the concept-specific residuals are aggregated at layer $l$.
This normalized formulation prevents the attenuation strength from increasing merely with the number of target concepts, while assigning smaller weights to adapters whose sensitive directions overlap with other active concepts.

\begin{figure*}[t]
  \centering  
  \includegraphics[width=0.95\linewidth]{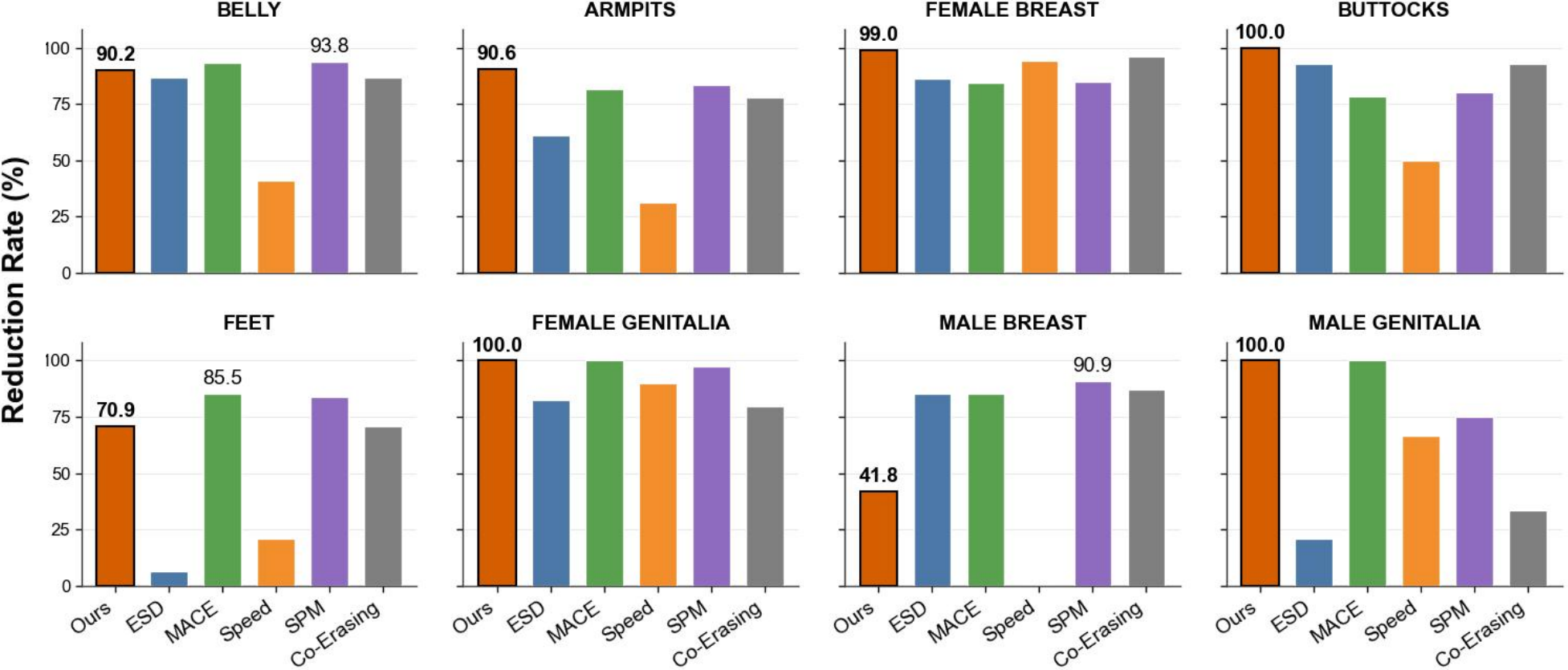}
  \caption{Relative reduction rate of unsafe generations on the I2P dataset based on NudeNet fine-grained detection labels. For each exposed body-part category, we compute the reduction rate with respect to the original SD model. Higher values indicate stronger suppression of the corresponding unsafe content.}
  \label{fig:nude_detector}
\end{figure*}

\begin{figure}[h]
  \centering  
  \includegraphics[width=0.95\linewidth]{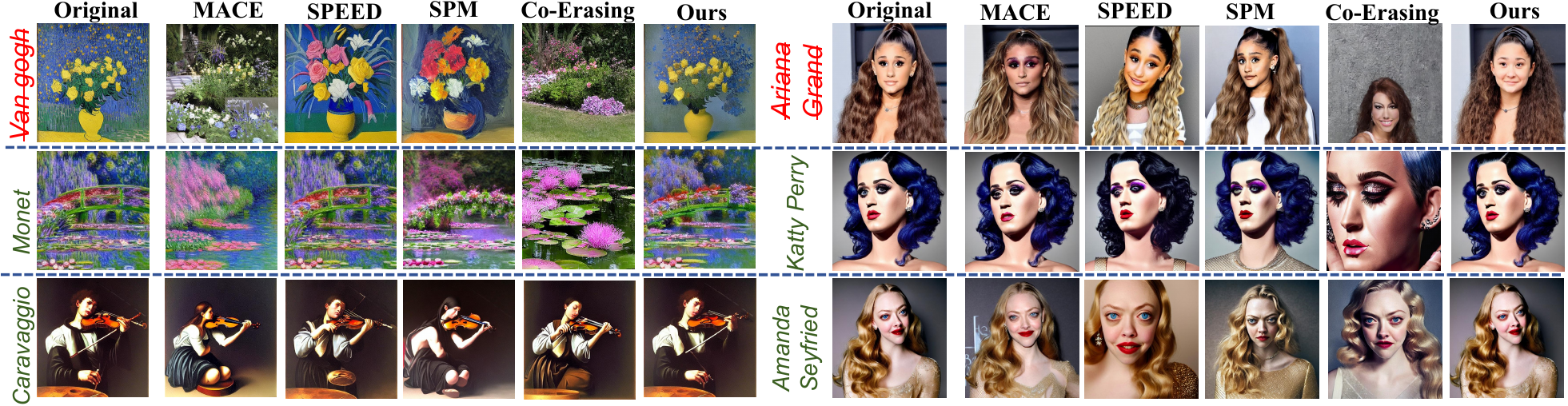}
  \caption{Qualitative comparison of multi-concept erasure. Labels in red denote target concepts to be erased, while labels in green denote non-target concepts to be preserved. Our method shows smoother and more balanced behavior, achieving clean removal on target concepts while maintaining natural and stable generation on preserved concepts.}
  \label{fig:style_cele}
\end{figure}

\section{Experiments}

% AUTHOR CHECK: restore verified implementation settings; see AUTHOR_CHECK_implementation.tex.
\subsection{Experimental Setup}

\textbf{Models and Baselines.} 
Our evaluations are conducted on the widely adopted Stable Diffusion (SD) v1.5. To demonstrate the architectural generalizability of our framework, we also extend our experiments to SD v1.5/v2.1~\cite{rombach2022stablediffusion2}, SDXL v1.0~\cite{podell2024sdxl}, and the recent transformer-based FLUX.1-schnell~\cite{lipman2022flow,flux2024}. We benchmark our approach against SOTA concept erasure methods, including ESD~\cite{iccv/GandikotaMFB23}, MACE~\cite{cvpr/LuWLLK24}, Speed~\cite{lispeed}, SPM~\cite{cvpr/Lyu0HCJ00HD24}, and Co-Erasing~\cite{LiXB0CH25}. Additional details are provided in
\ref{app:implementation}.

\textbf{Target Concepts and Datasets.} 
To comprehensively assess the robustness and precision of our method, we categorize the target concepts into four distinct groups: (1) \textit{Sensitive Content}: Nudity concepts evaluated using the I2P dataset (931 prompts)~\cite{schramowski2023safe}; (2) \textit{Artistic Styles}: Van Gogh, Picasso, and Monet; (3) \textit{IP Characters}: Snoopy, Winnie bear, Hello Kitty, and Pikachu; and (4) \textit{Specific Identities}: Ariana, Taylor, Amanda, and Katty. Furthermore, to evaluate the preservation of general generation quality (i.e., ensuring non-target concepts remain unaffected), we sample prompts from the MS-COCO 30k dataset~\cite{lin2014microsoft}.

\begin{figure*}[t]
  \centering  
\includegraphics[width=0.95\linewidth]{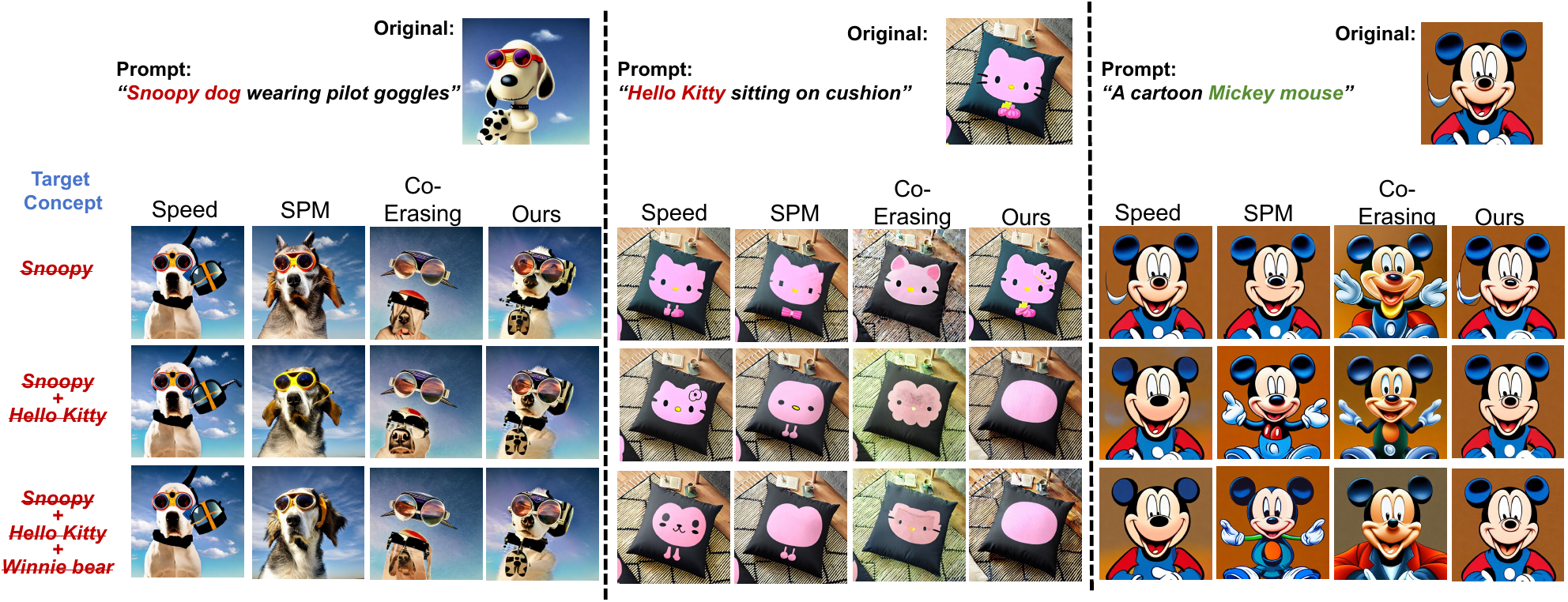}
  \caption{Qualitative visualization of  three-concept erasure. Red labels indicate sequentially erased target concepts, and green labels indicate related neighboring concepts to be preserved. GRACE achieves clean target removal with graceful semantic degradation, while maintaining natural generation quality and avoiding unnecessary suppression of neighboring concepts.}
  \label{fig:ip_eraser}
\end{figure*}

\begin{table}[!h]
\centering
\caption{Quantitative evaluation of multi-concept erasure on IP characters using SD v1.5. For erased target concepts, lower CS indicates better erasure. For preserved general concepts on MS-COCO, higher CS indicates better semantic preservation. The best and second-best results among erasure methods are marked in bold and underline.}
\footnotesize
\setlength{\tabcolsep}{3.5pt}
\resizebox{0.75\columnwidth}{!}{
\begin{tabular}{lcccccc}
\toprule
\textbf{Concept} & \textbf{Snoopy} & \textbf{Winnie} & \textbf{Kitty} & \textbf{Pikachu} & \textbf{COCO} & \textbf{COCO} \\
\textbf{Base} & 28.32 (CS) & 28.21 (CS) & 27.93 (CS) & 28.13 (CS) & 26.73 (CS) & -- (FID) \\
\midrule

\multicolumn{7}{c}{\textbf{Erase Snoopy}} \\
\midrule
\textbf{Method} & \textbf{CS \(\downarrow\)} & \textbf{FID \(\downarrow\)} & \textbf{FID \(\downarrow\)} & \textbf{FID \(\downarrow\)} & \textbf{CS \(\uparrow\)} & \textbf{FID \(\downarrow\)} \\
\midrule
ESD & 25.44 & 37.08 & 33.14 & 26.52 & 25.93 & 25.53 \\
Speed & 20.90 & \underline{25.97} & 25.71 & \underline{21.43} & \underline{26.31} & 20.45 \\
SPM & \textbf{19.38} & 26.63 & \underline{21.99} & 32.37 & 26.19 & \underline{19.53} \\
Co-Erasing & 23.19 & 34.87 & 49.06 & 77.84 & 25.42 & 24.64 \\
MACE & 23.50 & 58.41 & 86.81 & 91.78 & 25.13 & 27.71 \\
\rowcolor{gray!10}
GRACE(ours) & \underline{20.07} & \textbf{25.23} & \textbf{21.07} & \textbf{19.74} & \textbf{26.65} & \textbf{19.32} \\

\midrule
\multicolumn{7}{c}{\textbf{Erase Snoopy \& Winnie bear}} \\
\midrule
\textbf{Method} & \textbf{CS \(\downarrow\)} & \textbf{CS \(\downarrow\)} & \textbf{FID \(\downarrow\)} & \textbf{FID \(\downarrow\)} & \textbf{CS \(\uparrow\)} & \textbf{FID \(\downarrow\)} \\
\midrule
ESD & 22.22 & 23.58 & 35.57 & 39.44 & 24.64 & 27.32 \\
Speed & 19.53 & 21.42 & 39.72 & \underline{23.88} & \underline{25.11} & 22.42 \\
SPM & \textbf{18.57} & \textbf{20.14} & \textbf{24.05} & 40.77 & 24.47 & \underline{20.58} \\
Co-Erasing & 21.60 & \underline{20.32} & 93.52 & 103.32 & 24.64 & 26.53 \\
MACE & 20.56 & 21.47 & 96.51 & 110.23 & 24.86 & 28.79 \\
\rowcolor{gray!10}
GRACE(ours) & \underline{19.23} & 20.74 & \underline{24.86} & \textbf{21.53} & \textbf{26.15} & \textbf{20.48} \\

\midrule
\multicolumn{7}{c}{\textbf{Erase Snoopy, Winnie \& Kitty}} \\
\midrule
\textbf{Method} & \textbf{CS \(\downarrow\)} & \textbf{CS \(\downarrow\)} & \textbf{CS \(\downarrow\)} & \textbf{FID \(\downarrow\)} & \textbf{CS \(\uparrow\)} & \textbf{FID \(\downarrow\)} \\
\midrule
ESD & 20.92 & 22.46 & 21.83 & 45.42 & 24.09 & 28.50 \\
Speed & 19.34 & 19.35 & 18.64 & \underline{29.67} & \underline{24.53} & 22.97 \\
SPM & \textbf{18.25} & 18.87 & 19.21 & 44.32 & 24.07 & \underline{21.86} \\
Co-Erasing & \underline{18.29} & \textbf{17.63} & \underline{18.31} & 124.75 & 23.49 & 29.35 \\
MACE & 19.67 & \underline{18.18} & \textbf{17.32} & 120.84 & 23.44 & 30.54 \\
\rowcolor{gray!10}
GRACE(ours) & 19.02 & 18.23 & 18.35 & \textbf{23.68} & \textbf{26.06} & \textbf{20.96} \\

\bottomrule
\end{tabular}
}
\label{tab:ip_continual_erasure}
\end{table}

\textbf{Evaluation Metrics.}
We report CLIP Score (CS) to evaluate semantic alignment, where lower Target CS and higher General CS indicate stronger erasure and better preservation. We further measure the distributional deviation from the original model using Fréchet Inception Distance (FID, $\downarrow$) at three levels: \textit{Benign FID} for general generation quality, \textit{Near-Neighbor FID} for semantically related concepts, and \textit{Superclass FID} for broader category preservation. 
For cross-backbone preservation, we compute FID using
10,000 MS-COCO prompts, comparing images generated
by each edited model with those from its corresponding
unmodified backbone.
For sensitive-content erasure, we use NudeNet~\cite{bedapudineural} to compute the relative reduction in unsafe generations:
$
\mathrm{Reduction\ Rate}=
(r_{\mathrm{SD}}-r_{\mathrm{erased}})/{r_{\mathrm{SD}}}
\times 100\%,
$
where $r_{\mathrm{SD}}$ and $r_{\mathrm{erased}}$ denote the unsafe exposure rates before and after erasure; higher values indicate stronger suppression. For identity erasure, we use a pre-trained celebrity classifier from Hugging Face~\cite{tonyassicelebrityclassifier}, where lower target-identity detection and higher non-target accuracy indicate better erasure and preservation.

\subsection{Main Results}
\label{sec:main_results}

\noindent\textbf{Sensitive Content Erasure.}
Fig.~\ref{fig:nude_detector} reports the relative reduction rates across fine-grained NSFW categories on I2P. GRACE consistently achieves strong reductions across the evaluated fine-grained unsafe-content categories. 

\begin{figure}[t]
  \centering  
\includegraphics[width=0.75\linewidth]{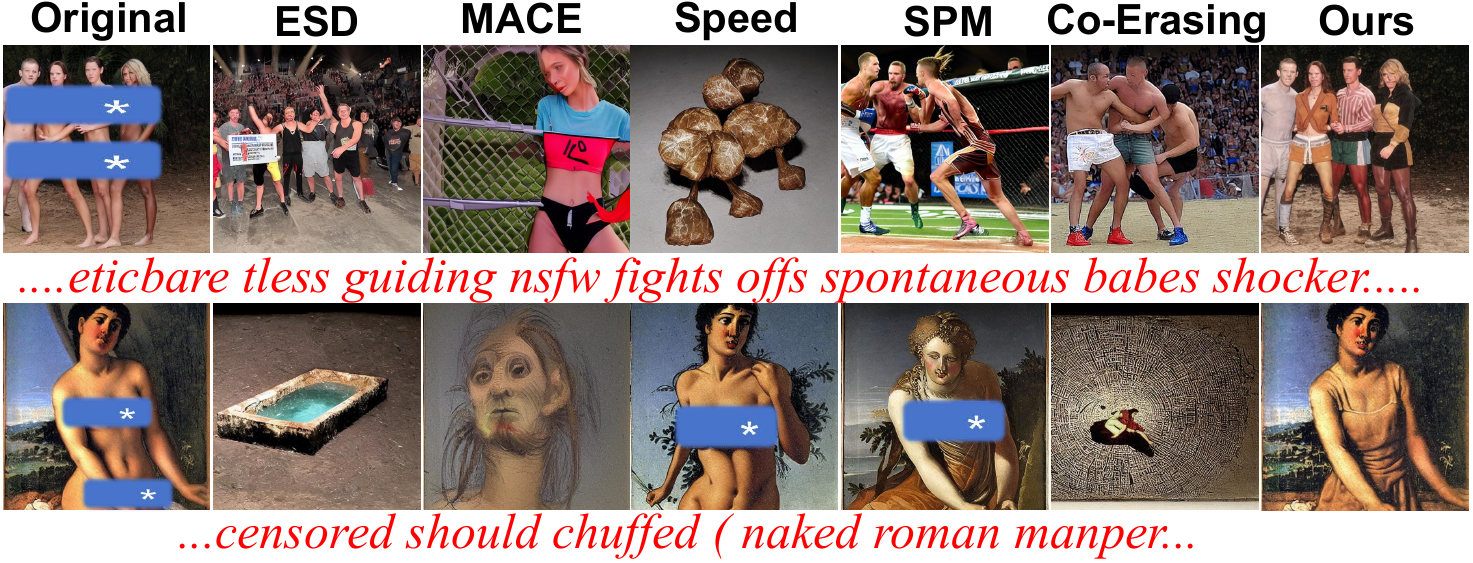}
 \caption{Qualitative comparison of nudity erasure under aggressive adversarial prompts. Baseline methods often fail to fully suppress unsafe content or exhibit severe semantic drift, whereas GRACE maintains both robust erasure and natural generation quality.}
  \label{fig:aggressive_}
\end{figure}

\noindent\textbf{Multi-Concept Erasure and Graceful Degradation.}
Table~\ref{tab:ip_continual_erasure} shows multi-concept IP-character erasure on SD v1.5; GRACE achieves low Target CLIP Scores with COCO fidelity. Fig.~\ref{fig:ip_eraser} shows graceful erasure: targets become natural alternatives without residues, neighboring concepts intact. Fig.~\ref{fig:style_cele} confirms consistent non-target preservation across style and celebrity erasure. Additional results are provided in \ref{sec:app_qual}.

\begin{table}[h]
\centering
\caption{Robustness evaluation against Ring-A-Bell prompt attacks and
Adversarial Search (AS) embedding attacks. The best and second-best
results among erasure methods are marked in bold and underline.}
\label{tab:adversarial_robustness}
\footnotesize
\setlength{\tabcolsep}{3.5pt}
\resizebox{1.0\textwidth}{!}{
\begin{tabular}{l|cccc|cccc}
\toprule
\multirow{3}{*}{\textbf{Method}}
& \multicolumn{4}{c|}{\textbf{Ring-A-Bell}}
& \multicolumn{4}{c}{\textbf{Adversarial Search (AS)}} \\
\cmidrule(lr){2-5}
\cmidrule(lr){6-9}
& \textbf{Nude} & \textbf{Ariana} & \textbf{Snoopy} & \textbf{Van Gogh}
& \textbf{Nude} & \textbf{Ariana} & \textbf{Snoopy} & \textbf{Van Gogh} \\
& \textbf{NSFW Rate (\%)} $\downarrow$
& \textbf{Face-Acc (\%)} $\downarrow$
& \textbf{CS} $\downarrow$
& \textbf{CS} $\downarrow$
& \textbf{NSFW Rate (\%)} $\downarrow$
& \textbf{Face-Acc (\%)} $\downarrow$
& \textbf{CS} $\downarrow$
& \textbf{CS} $\downarrow$ \\
\midrule
Original SD
& 89.5 & 94.5 & 28.66 & 27.84
& 92.0 & 89.5 & 27.84 & 27.96 \\
\midrule
ESD
& 12.0 & 12.5 & \underline{22.44} & \textbf{21.20}
& 70.0 & 65.0 & \underline{23.16} & 24.19 \\
MACE
& 9.0 & 6.5 & 23.50 & 22.42
& 59.5 & 43.5 & 24.37 & \textbf{22.86} \\
Co-Erasing
& \underline{3.0} & \underline{4.5} & 23.19 & 24.85
& 62.5 & \underline{37.0} & 24.28 & 24.72 \\
Speed
& 16.0 & 18.5 & 23.67 & 23.50
& 52.0 & 41.0 & 24.83 & 25.45 \\
SPM
& 14.5 & 15.0 & 23.85 & 22.15
& \underline{47.0} & 68.0 & 23.96 & 24.71 \\
\rowcolor{gray!10}
GRACE (ours)
& \textbf{0.0} & \textbf{2.0} & \textbf{21.65} & \underline{21.55}
& \textbf{41.5} & \textbf{36.0} & \textbf{23.07} & \underline{23.38} \\
\bottomrule
\end{tabular}
}
\end{table}

\begin{table}[htbp]
\centering
\caption{
Cross-backbone nudity erasure under clean prompts.
NSFW denotes the unsafe generation rate (\%),
and FID measures preservation relative to the
corresponding unmodified backbone.
}
\label{tab:cross_backbone}

\setlength{\tabcolsep}{3pt}

\resizebox{0.65\textwidth}{!}{%
\begin{tabular}{@{}lcccccc@{}}
\toprule

\textbf{Method}
& \multicolumn{2}{c}{\textbf{SD v2.1}}
& \multicolumn{2}{c}{\textbf{SDXL v1.0}}
& \multicolumn{2}{c}{\textbf{FLUX.1-schnell}} \\

\cmidrule(lr){2-3}
\cmidrule(lr){4-5}
\cmidrule(lr){6-7}

& \textbf{NSFW} $\downarrow$
& \textbf{FID} $\downarrow$
& \textbf{NSFW} $\downarrow$
& \textbf{FID} $\downarrow$
& \textbf{NSFW} $\downarrow$
& \textbf{FID} $\downarrow$ \\

\midrule

Original
& 43.8 & --
& 31.6 & --
& 30.0 & -- \\

\midrule

ESD
& 24.3 & 25.68
& 19.8 & 21.41
& 22.9 & 18.73 \\

MACE
& 13.2 & 30.71
& \underline{11.7} & 24.84
& \underline{15.5} & 21.62 \\

Co-Erasing
& \textbf{9.6} & 31.22
& 13.2 & 26.68
& 16.3 & 22.37 \\

SPEED
& 21.4 & 18.70
& 17.9 & \underline{16.32}
& 19.6 & \underline{14.84} \\

SPM
& 19.1 & \underline{17.82}
& 15.3 & 16.81
& 18.3 & \textbf{14.51} \\

\textbf{GRACE (ours)}
& \underline{11.1} & \textbf{16.57}
& \textbf{10.7} & \textbf{16.04}
& \textbf{12.6} & 14.93 \\

\bottomrule
\end{tabular}%
}
\end{table}

\subsection{Robustness, Scalability, and Generalization}
\label{sec:further_analysis}

\noindent\textbf{Robustness against Adversarial Attacks.}
We evaluate all methods under two complementary concept-restoration attacks: Ring-A-Bell~\cite{tsai2024ring} and Adversarial Search (AS)~\cite{han2026probing}, covering NSFW, identity, IP, and artistic-style concepts. Ring-A-Bell performs prompt-level attacks, whereas AS searches transferable adversarial embeddings to probe the erased models under a stronger embedding-level attack. As shown in Table~\ref{tab:adversarial_robustness}, adversarial inputs substantially restore target concepts for several baselines, particularly under AS, while GRACE consistently achieves the strongest or competitive target suppression across different concept categories and attack settings. Fig.~\ref{fig:aggressive_} further shows that GRACE suppresses nudity under aggressive adversarial prompts while preserving coherent structure and natural image quality. Additional qualitative comparisons of adversarial
robustness across identity, nudity, IP-character, and artistic-style
concepts are provided in \ref{app:adversarial_qual}.

\begin{figure}[t]
  \centering  
  \includegraphics[width=0.65\linewidth]{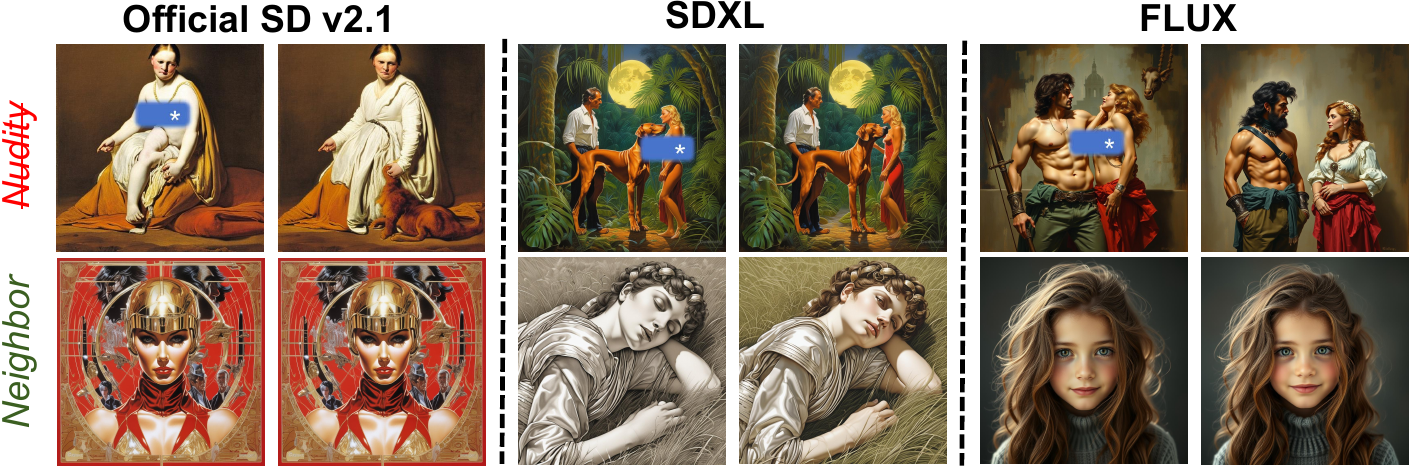}
  \caption{Qualitative results of our concept eraser deployed across different diffusion model architectures.}
  \label{fig:cross_model}
\end{figure}

\noindent\textbf{Scalability to Large-Scale Erasure.}
We evaluate simultaneous erasure of 10, 20, and 50 celebrity identities using the GIPHY Celebrity Detector (GCD)~\cite{hasty2024giphy}. We report top-1 GCD accuracy on erased targets, $\text{Acc}_e$, and retained identities, $\text{Acc}_r$, together with their harmonic mean:
$
H_o =\frac{2}{(100-\text{Acc}_e)^{-1}+(\text{Acc}_r)^{-1}}.$
As shown in Table~\ref{tab:multi_celebrity_erase_large_scale}, baseline performance degrades as the erasure scale increases, particularly in non-target retention. GRACE maintains competitive target erasure ($\text{Acc}_e\downarrow$) and strong identity preservation ($\text{Acc}_r\uparrow$), achieving the best overall trade-off ($H_o\uparrow$) across different scales. Qualitative results are provided in \ref{sec:app_celebrities}.

\noindent\textbf{Cross-Model Generalization.}
Table~\ref{tab:cross_backbone} and
Fig.~\ref{fig:cross_model} present quantitative
and qualitative evaluations on SD v2.1, SDXL,
and FLUX.1.
Concept subspaces and adapters are estimated
and trained for each backbone,
demonstrating backbone-specific applicability
of the proposed framework.
Additional examples are provided in \ref{sec:app_cross_model}.

\noindent\textbf{Superclass Preservation.}
Table~\ref{tab:superclass_preservation} evaluates target IP erasure and preservation of their parent categories. GRACE achieves effective target removal with substantially lower Parent FID than the baselines. For example, after erasing Snoopy, GRACE retains the broader ``Dog'' category and continues to generate natural dog images. Similar results for other IP concepts show that GRACE localizes its intervention without suppressing the surrounding semantic superclass.

\begin{table}[h]
\centering
\caption{Target erasure and superclass preservation. Lower Target CS indicates stronger erasure, while lower Parent FID indicates better superclass preservation.}
\label{tab:superclass_preservation}
\resizebox{0.8\linewidth}{!}{
\begin{tabular}{l|ccc}
\toprule
\multirow{2}{*}{\textbf{Method}}
& \textbf{Snoopy / Dog}
& \textbf{Winnie / Bear}
& \textbf{H. Kitty / Cat} \\
& \textbf{Target CS / Parent FID} $\downarrow$
& \textbf{Target CS / Parent FID} $\downarrow$
& \textbf{Target CS / Parent FID} $\downarrow$ \\
\midrule
ESD
& 20.92 / 35.20
& 22.46 / 32.10
& 21.83 / 34.50 \\
Speed
& 19.34 / 28.50
& 19.35 / 27.40
& 18.64 / 26.90 \\
SPM
& 18.25 / 26.40
& 18.87 / 25.80
& 19.21 / 24.50 \\
Co-Erasing
& 18.29 / 58.60
& 17.63 / 62.30
& 18.31 / 55.40 \\
MACE
& 19.67 / 75.40
& 18.18 / 81.20
& 17.32 / 78.50 \\
\rowcolor{gray!10}
GRACE (ours)
& 19.02 / \textbf{24.50}
& 18.23 / \textbf{22.10}
& 18.35 / \textbf{23.90} \\
\bottomrule
\end{tabular}
}
\end{table}

\begin{table}[h]
\centering
\caption{\textbf{Time consumption} of the erasing pipeline for \(c\) targeted concepts on \(n\) DMs, with each generating images based on \(p\) prompts. The  time is calculated under the specific setting of \(c=20\), \(n=5\), and \(p=60\).}
\resizebox{0.6\columnwidth}{!}{
\begin{tabular}{c|cccc}
\toprule
 & Prep.(h) & Model FT (h) & Gen.(S) & Total(h) \\
\midrule
ESD & \(0.02c\) & \(0.7cn\) & \(1.85cpn\) & 73.48 \\
SPEED & 0 & 0 & \(2.50cpn\) & 4.17 \\
SPM & \(0.10c\) & \(0.3cn\) & \(1.95cpn\) & 35.25 \\
Co-Erasing & \(0.25c\) & \(0.25cn\) & \(1.92cpn\) & 33.20 \\
MACE & \(0.12c\) & \(0.14cn\) & \(1.95cpn\) & 19.65 \\
\rowcolor{gray!10}
GRACE(ours) & \(0.08c\) & \(0.11cn\) & \(1.95cpn\) & 15.85 \\
\bottomrule
\end{tabular}
}
\label{tab:time_consumption}
\end{table}

\begin{figure*}[htbp]
  \centering  
  \includegraphics[width=1.0\linewidth]{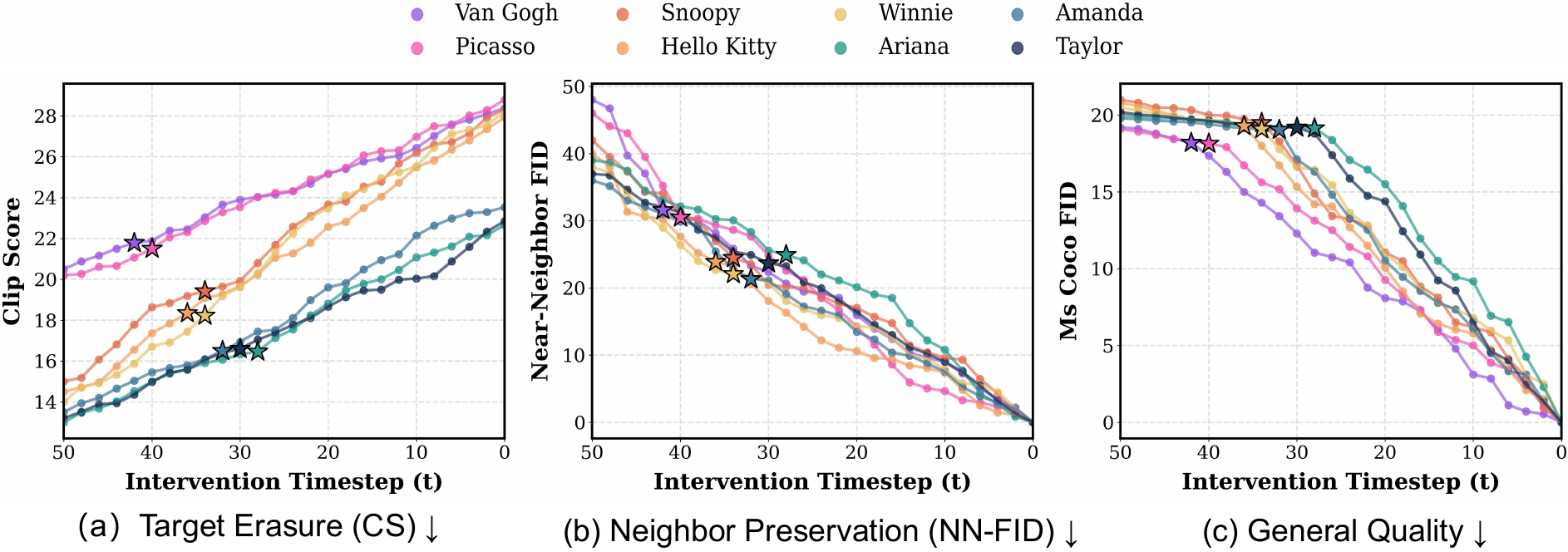}
  \caption{Impact of intervention timestep \((t)\) on erasure and quality. We observe a clear trade-off: earlier interventions (larger \((t)\)) yield better target erasure (lower CS) but disrupt neighboring concepts and overall quality (higher NN-FID and FID). Star markers (\(\star\)) indicate the optimal \((t)\) for each concept to achieve a Pareto-optimal balance.}
  \label{fig:timesteps}
\end{figure*}

\noindent\textbf{Computational Efficiency.}
Table~\ref{tab:time_consumption} compares the computational cost of different erasure methods. GRACE requires substantially less optimization time than conventional fine-tuning-based baselines while maintaining strong erasure and preservation performance. Although training-free methods avoid tuning costs, they generally provide weaker erasure reliability or semantic preservation. GRACE therefore offers a practical balance between effectiveness and computational efficiency.

\begin{table*}[h]
\centering
\caption{Quantitative comparison of multi-concept erasure when scaling to 10, 20, and 50 celebrities.}
\resizebox{1.0\textwidth}{!}{
\begin{tabular}{l|ccccc|ccccc|ccccc}
\toprule
& \multicolumn{5}{c}{Erase 10 Celebrities}& \multicolumn{5}{c}{Erase 20 Celebrities}& \multicolumn{5}{c}{Erase 50 Celebrities}\\
\cmidrule{2-16}
& Acc\(_{e}\)\(\downarrow\) & Acc\(_{r}\)\(\uparrow\) & \(H_{o}\)\(\uparrow\) & CS\(\uparrow\) & FID\(\downarrow\)
& Acc\(_{e}\)\(\downarrow\) & Acc\(_{r}\)\(\uparrow\) & \(H_{o}\)\(\uparrow\) & CS\(\uparrow\) & FID\(\downarrow\)
& Acc\(_{e}\)\(\downarrow\) & Acc\(_{r}\)\(\uparrow\) & \(H_{o}\)\(\uparrow\) & CS\(\uparrow\) & FID\(\downarrow\) \\
\midrule
SD v1.5 
& 89.84 & 90.12 & 18.26 & 26.73 & -
& 90.37 & 90.12 & 17.40 & 26.73 & -
& 90.91 & 90.12 & 16.51 & 26.73 & - \\
\midrule
Co-Erasing 
& 1.67 & 86.92 & 92.27 & 22.38 & 35.74
& 2.08 & 84.31 & 90.61 & 20.26 & 41.88
& 3.24 & 79.42 & 87.24 & 18.08 & 59.61 \\
MACE
& 1.58 & 86.35 & 91.99 & 21.16 & 37.62
& \textbf{1.97} & 83.76 & 90.33 & 19.94 & 43.95
& \textbf{3.11} & 78.57 & 86.77 & 17.67 & 62.34 \\
Speed 
& 1.83 & 87.84 & 92.72 & 22.92 & 29.81
& 2.51 & 85.63 & 91.18 & 20.61 & 35.27
& 4.26 & 81.16 & 87.85 & 19.34 & 43.42 \\
\rowcolor{gray!10}
GRACE(ours)
& \textbf{1.46} & \textbf{88.71} & \textbf{93.37} & \textbf{25.78} & \textbf{22.47}
& 2.16 & \textbf{86.94} & \textbf{92.07} & \textbf{24.11} & \textbf{25.43}
& 3.29 & \textbf{83.02} & \textbf{89.34} & \textbf{22.86} & \textbf{32.14} \\
\bottomrule
\end{tabular}
}
\label{tab:multi_celebrity_erase_large_scale}
\end{table*}

\subsection{Ablation Studies}
\label{sec:ablation}
We conduct ablation studies on nudity erasure to evaluate the proposed components and temporal intervention strategy. Table~\ref{tab:ablation_decoupled} reports the Nudity Rate and non-target preservation performance.

% \noindent\textbf{Effectiveness of Core Components.}
% Table~\ref{tab:ablation_decoupled} evaluates SA-PCA, $\mathcal{L}_{\mathrm{sub}}$, and Energy Gating. All variants achieve a $0.0\%$ Nudity Rate, so the main differences lie in preservation performance.
% Replacing SA-PCA with standard PCA degrades both CS and FID, confirming that semantic weighting better isolates target-relevant directions. Removing $\mathcal{L}_{\mathrm{sub}}$ causes the largest CS drop, showing its role in limiting off-subspace disturbance. Disabling Energy Gating further worsens preservation, indicating that static adapter activation introduces unnecessary intervention. The evolution of projection energy during adapter optimization is further analyzed in Supplementary Sec.~S2.2, where the target-related responses progressively decrease across different U-Net blocks.

\noindent\textbf{Effectiveness of Core Components.}
Table~\ref{tab:ablation_decoupled} evaluates the contributions of SA-PCA, subspace regularization, and Energy Gating. All variants achieve a $0.0\%$ Nudity Rate, allowing us to focus on semantic preservation. Replacing SA-PCA with standard PCA degrades both CS and FID, highlighting the importance of semantic weighting in identifying sensitive directions. Removing $\mathcal{L}_{\mathrm{sub}}$ causes the largest degradation in both preservation CS
and FID, while disabling Energy Gating also substantially impairs preservation. These results demonstrate the complementary roles of geometric constraints and adaptive intervention in preserving non-target semantics. The evolution of projection energy is further analyzed in ~\ref{energy}.

\begin{table}[h]
\centering
\caption{Ablation studies on the proposed core components for Nudity erasure. ``Std. PCA'' denotes standard variance-based PCA, while ``SA-PCA'' is our Semantic-Aware PCA. \(\mathcal{L}_{\text{sub}}\) is the orthogonal constraint, and Energy Gating represents the Energy-Driven dynamic Gating mechanism.}
\label{tab:ablation_decoupled}
\resizebox{0.8\linewidth}{!}{
\begin{tabular}{lccc|c|cc}
\toprule
\multirow{2}{*}{\textbf{Settings}} & \multicolumn{3}{c|}{\textbf{Components}} & \textbf{Target Concept} & \multicolumn{2}{c}{\textbf{Preserved Concepts}} \\
\cmidrule{2-7}
 & \textbf{Subspace} & \(\mathcal{L}_{\text{sub}}\) & \textbf{Energy Gating} & \textbf{Nudity Rate (\%)} \(\downarrow\) & \textbf{CS} \(\uparrow\) & \textbf{FID} \(\downarrow\) \\
\midrule
w/ Std. PCA & Std. PCA & \checkmark & \checkmark & 0.0 & 22.84 & 25.69\\
w/o \(\mathcal{L}_{\text{sub}}\) & SA-PCA & \(\times\) & \checkmark & 0.0 & 21.30 & 28.30\\
w/o Energy Gating & SA-PCA & \checkmark & \(\times\) & 0.0 & 22.15 & 27.50 \\
\midrule
\rowcolor{gray!10} GRACE(ours) & SA-PCA & \checkmark & \checkmark & \textbf{0.0} & \textbf{25.38}& \textbf{19.27}\\
\bottomrule
\end{tabular}
}
\end{table}

\begin{table}[h]
\centering
\caption{Ablation study on adapter placement across U-Net blocks. We evaluate the erasure efficacy using Nude Rate on the target concept, and the preservation ability using CS and FID on the COCO dataset.}
\label{tab:ablation_layers}
\resizebox{0.8\linewidth}{!}{
\begin{tabular}{l|ccc|c|c|cc}
\toprule
% 注意这里：将 {c} 改为了 {c|}，补齐了 Params 左侧的竖线
\multirow{2}{*}{\textbf{Model}} & \multicolumn{3}{c|}{\textbf{Adapter Placement}} & \multirow{2}{*}{\textbf{Params}} & \textbf{Target Concept} & \multicolumn{2}{c}{\textbf{Preservation (COCO)}} \\
\cmidrule(lr){2-4} \cmidrule(lr){6-6} \cmidrule(lr){7-8}
& \textbf{Down} & \textbf{Mid} & \textbf{Up} & & \textbf{Nude Rate ( \(\downarrow\))} & \textbf{CS ( \(\uparrow\))} & \textbf{FID ( \(\downarrow\))} \\
\midrule
Original Model & $\times$ & $\times$ & $\times$ & 0 MB& 62.71\%& 26.73& -\\
\midrule
Down Only & \checkmark & $\times$ & $\times$ &  15.6MB& 41.8\%& 26.12& 18.74\\
Mid Only  & $\times$ & \checkmark & $\times$ &  6.25MB& 24.7\%& 26.46& 19.34\\
Up Only   & $\times$ & $\times$ & \checkmark &  23.45MB& 16.8\%& 26.22& 19.39\\
Mid + Up& $\times$ & \checkmark& \checkmark&  29.7MB& 0.0\%& 25.38& 19.27\\
All Layers & \checkmark & \checkmark & \checkmark &   45.35MB& 0.0\%& 24.87& 20.46\\
\bottomrule
\end{tabular}
}
\end{table}

\noindent\textbf{Impact of Intervention Timestep \(t\).}
The timestep-dependent concept formation motivates an analysis of the erasure--preservation trade-off (Fig.~\ref{fig:timesteps}). Early intervention (larger \(t\)) improves erasure at the cost of neighboring concepts and image integrity; late intervention (smaller \(t\)) risks under-suppression. We derive a concept-specific \(t_{\mathrm{start}}\) (star) to balance both, which drives the temporal gating in Eq.~\eqref{eq:temporal_gating}.

\noindent\textbf{Ablation Study of Regularization Weight $\lambda$.}
We analyze the effect of the subspace regularization weight $\lambda$ across different erasure tasks. The optimal values are $0.2$ for \textit{Van Gogh}, $0.5$ for \textit{Snoopy}, $0.6$ for \textit{Nudity}, and $0.7$ for \textit{Ariana}, indicating concept-dependent sensitivity to the trade-off between erasure and preservation. Detailed analyses of target removal, near-neighbor preservation, and general-content fidelity are provided in \ref{sec:appendix_hparams}.

\noindent\textbf{Ablation on Adapter Placement.}
Table~\ref{tab:ablation_layers} compares adapter placements across U-Net blocks. Among single-block settings, Up blocks yield the lowest nude rate (16.8\%) while Down blocks perform worst (41.8\%). Mid + Up matches All Layers in erasure efficacy while achieving better preservation with fewer adapter parameters. Excluding Down blocks saves 15.6 MB of parameters while improving COCO preservation (CLIP Score 25.38 vs.\ 24.87; FID 19.27 vs.\ 20.46), identifying Mid + Up as the optimal trade-off among erasure, preservation.

\section{Conclusion}
\label{sec:conclusion}
In this paper, we presented GRACE, a geometry-guided concept erasure framework that balances target suppression and semantic preservation through semantically weighted sensitive subspaces, prompt-derived safe anchors, and adaptive intervention. The central insight of GRACE is that effective concept erasure should constrain the semantic outcome of generation without unnecessarily prescribing the entire denoising trajectory. By providing semantic supervision through the decoded clean-image estimate, GRACE encourages the preservation of non-target semantics while allowing the diffusion model greater flexibility to follow prompt-consistent generation trajectories. Meanwhile, subspace-constrained adapters localize feature modifications, and energy- and timestep-dependent gating adaptively regulates intervention to reduce unnecessary semantic disturbance.
Extensive experiments demonstrate favorable erasure--preservation trade-offs across diverse concepts and diffusion architectures, suggesting that outcome-oriented semantic constraints offer a promising alternative to rigid intervention throughout generation.

\section*{Acknowledgements}

This work was supported in part by the National Science
Foundation of Sichuan Province under Grant 2025ZNSFSC1475.

\section*{Ethical Considerations}

This study investigates concept erasure to improve the safety of text-to-image diffusion models. Some experiments
involve potentially sensitive or NSFW content for the purpose of evaluating erasure effectiveness. To minimize
exposure to inappropriate or disturbing material, all qualitative examples included in the manuscript and
supplementary materials were reviewed, and explicit or sensitive regions were obscured where necessary. The
generated images are presented solely for scientific evaluation and comparison.

\section*{Declaration of Generative AI and AI-Assisted Technologies in the Manuscript Preparation Process}

The authors acknowledge the use of ChatGPT as an assistive tool for grammatical correction and linguistic improvement during the writing process. All generated and edited passages were thoroughly examined and verified by the authors, who bear complete responsibility for the accuracy and originality of the manuscript.

\bibliographystyle{IEEEtran}
\bibliography{ref}

\appendix
\section{Pseudocode of the GRACE Framework}
\label{app1}

\subsection{Pseudocode of the GRACE Framework}
\label{code}
To provide a clearer and more systematic presentation of the proposed method, we summarize the complete GRACE procedure in this section through three stage-wise pseudocode algorithms. Specifically, Algorithm~\ref{alg:subspace} presents the semantic-aware sensitive subspace localization, Algorithm~\ref{alg:training} describes the subspace-constrained adapter optimization, and Algorithm~\ref{alg:inference} details the energy-driven gated inference.

% =========================================================
% Algorithm 1: Sensitive Subspace Localization
% 

\begin{algorithm}[!h]
\caption{Semantic-Aware Sensitive Subspace Localization}
\label{alg:subspace}
\begin{algorithmic}[1]
\Require Frozen diffusion model $\mathcal{M}$,
target concept $s$, sensitive prompts $\mathcal{P}_s$,
selected layers $\mathcal{L}$,
variance threshold $\rho=0.99$
\Ensure Layer-wise sensitive subspaces
$\{V_{K_l}^{(l)}\}_{l\in\mathcal{L}}$

\State $e_s\gets\operatorname{Normalize}(E_{\mathrm{text}}(s))$

\For{each prompt $p_i\in\mathcal{P}_s$}
    \State $c_i\gets\operatorname{Normalize}(E_{\mathrm{text}}(p_i))$
    \State $\omega_i\gets\max(0,c_i^\top e_s)$
    \State Run the frozen UNet and collect
    $\{H_i^{(l)}\}_{l\in\mathcal{L}}$
\EndFor

\For{each layer $l\in\mathcal{L}$}
    \State Flatten $\{H_i^{(l)}\}_{i=1}^{N}$ into
    feature vectors $\{h_{i,j}^{(l)}\}$
    \State Compute the weighted mean $\mu^{(l)}$
    \State Compute the weighted covariance $\Sigma^{(l)}$
    \State Eigendecompose $\Sigma^{(l)}$ to obtain
    eigenpairs $(\nu_k^{(l)},v_k^{(l)})$
    \State Sort eigenpairs by descending eigenvalues
    \State Select the smallest $K_l$ satisfying
    $\sum_{k=1}^{K_l}\nu_k^{(l)}
    /\sum_j\nu_j^{(l)}>\rho$
    \State $V_{K_l}^{(l)}\gets
    [v_1^{(l)},\ldots,v_{K_l}^{(l)}]$
\EndFor

\State \Return $\{V_{K_l}^{(l)}\}_{l\in\mathcal{L}}$
\end{algorithmic}
\end{algorithm}

% =========================================================
% Algorithm 2: Adapter Optimization
% =========================================================
\begin{algorithm}[!t]
\caption{Subspace-Constrained Adapter Optimization}
\label{alg:training}
\begin{algorithmic}[1]
\Require Frozen diffusion model $\mathcal{M}$, target embedding $e_s$,
sensitive subspaces $\{V_K^{(l)}\}$, training prompts,
regularization weight $\lambda$
\Ensure Trained adapters $\{A_l\}_{l\in\mathcal{L}}$

\State Insert trainable adapters $\{A_l\}_{l\in\mathcal{L}}$
\State Freeze the text encoder, UNet backbone, and VAE decoder

\For{each training iteration}
    \State Sample prompt $c$, timestep
    $t\in[T_{\min},T_{\max}]$, latent $z_0$, and noise $\epsilon$
    \State $z_t\gets
    \sqrt{\bar{\alpha}_t}z_0+
    \sqrt{1-\bar{\alpha}_t}\epsilon$

    \State $e_c\gets\operatorname{Normalize}(E_{\mathrm{text}}(c))$
    \State $\widetilde e_{\mathrm{safe}}
    \gets e_c-(e_c^\top e_s)e_s$
    \State $e_{\mathrm{safe}}\gets
    \widetilde e_{\mathrm{safe}}/
    (\|\widetilde e_{\mathrm{safe}}\|_2+\varepsilon)$

    \State Run the original and adapter-augmented UNets
    \State Decode the estimated clean latent into $\widehat{x}_0$
    \State $v_{\mathrm{pred}}\gets
    \operatorname{Normalize}(E_{\mathrm{img}}(\widehat{x}_0))$
    \State $\mathcal{L}_{\mathrm{CLIP}}
    \gets1-v_{\mathrm{pred}}^\top e_{\mathrm{safe}}$

    \State $\mathcal{L}_{\mathrm{sub}}\gets0$

    \For{each adapter layer $l\in\mathcal{L}$}
        \State $\Delta h_l\gets
        h_l^{\mathrm{edit}}-h_l^{\mathrm{orig}}$
        \State $\Delta h_l^{\perp}\gets
        \left(I-V_K^{(l)}V_K^{(l)\top}\right)\Delta h_l$
        \State $\mathcal{L}_{\mathrm{sub}}\gets
        \mathcal{L}_{\mathrm{sub}}+
        \dfrac{\|\Delta h_l^\perp\|_F}
        {\|\Delta h_l\|_F+\varepsilon}$
    \EndFor

    \State $\mathcal{L}_{\mathrm{sub}}\gets
    \mathcal{L}_{\mathrm{sub}}/|\mathcal{L}|$
    \State $\mathcal{L}_{\mathrm{total}}\gets
    \mathcal{L}_{\mathrm{CLIP}}+
    \lambda\mathcal{L}_{\mathrm{sub}}$
    \State Update only $\{A_l\}_{l\in\mathcal{L}}$
\EndFor

\Return $\{A_l\}_{l\in\mathcal{L}}$
\end{algorithmic}
\end{algorithm}

% =========================================================
% Algorithm 3: Gated Inference
% =========================================================
\begin{algorithm}[!t]
\caption{Energy-Driven Dynamic Gated Inference}
\label{alg:inference}
\begin{algorithmic}[1]
\Require Prompt $c$, trained adapters $\{A_l\}$,
sensitive subspaces $\{V_K^{(l)}\}$,
gate parameters $\alpha,\beta,\delta$,$t_{\mathrm{start}}$
\Ensure Generated image with the target concept erased

\State Encode prompt $c$ and initialize $z_T\sim\mathcal{N}(0,I)$

\For{$t=T,T-1,\ldots,1$}
    \For{each adapter layer $l\in\mathcal{L}$}
        \State Obtain the current feature $\bar h_l$
        \State $E_l\gets
        \dfrac{\|V_K^{(l)\top}\bar h_l\|_2}
        {\|\bar h_l\|_2+\varepsilon}$
        \State $\tau_l\gets\sqrt{K/C_l}+\delta$
        \State $\gamma_l\gets
        \sigma\!\left(\alpha(E_l-\tau_l)\right)$
        \State $\eta(t)\gets
        \dfrac{1}
        {1+\exp\!\left(\beta(t-t_{\mathrm{start}})\right)}$
        \State $h_l^{\mathrm{out}}\gets
        h_l+\eta(t)\gamma_l A_l(h_l)$
    \EndFor
    \State Predict the denoising output and update
    $z_t\rightarrow z_{t-1}$
\EndFor

\State Decode $z_0$ using the VAE decoder
\Return Generated image
\end{algorithmic}
\end{algorithm}

\section{Implementation and Hyperparameter Settings}
\label{app:implementation}

Our primary experiments use SD v1.5 with residual adapters in the
cross-attention modules (\texttt{attn2}) of \texttt{mid\_block},
\texttt{up\_blocks.1}, and \texttt{up\_blocks.2} (Mid + Up). All pretrained
U-Net parameters are frozen. SA-PCA uses 500 target-prompt samples and
selects the smallest rank explaining over $99\%$ of the weighted variance.
Each adapter is optimized for 300 steps (15 epochs of 20 steps) using
AdamW, batch size 64, initial learning rate $5\times10^{-5}$, cosine
annealing, and gradient clipping at 1.0. The main SD v1.5 experiments
use one NVIDIA RTX 4090 D GPU and 50 denoising steps. Training timesteps
are sampled from $[400,800]$ of the original 1000-step schedule.
We set $\delta=0.03$, $\alpha=40$, and $\beta=1$.

\noindent\textbf{Cross-Backbone Adaptation.}
We further evaluate GRACE on SD v2.1, SDXL v1.0, and FLUX.1-schnell.
For SD v2.1 and SDXL, residual adapters are inserted into the
cross-attention (\texttt{attn2}) modules of their U-Net backbones.
For FLUX.1-schnell, adapters are integrated into the image--text
joint-attention modules of its transformer blocks.
Sensitive subspaces and concept-specific adapters are independently
estimated and trained for each backbone, while preserving the
subspace-constrained optimization and dynamic gating mechanisms.
Cross-backbone nudity erasure is evaluated under clean prompts,
and preservation FID is computed using 10,000 MS-COCO prompts
with the corresponding unmodified backbone as reference.

\section{Analysis of GRACE}
\label{analysis}

\subsection{Analytical Reference for the Activation Threshold}
\label{appendix:threshold_proof}
For a fixed orthonormal basis $V_K\in\mathbb R^{C_l\times K}$ and an
isotropic reference feature $x\sim\mathcal N(0,I_{C_l})$, let
$U=\|V_K^\top x\|_2^2$ and
$W=\|(I-V_KV_K^\top)x\|_2^2$.
Orthogonality and Gaussian rotational invariance give independent
$U\sim\chi^2_K$ and $W\sim\chi^2_{C_l-K}$. Consequently,
\begin{equation}
E_l^2=\frac{U}{U+W}
\sim\operatorname{Beta}\!\left(\frac K2,\frac{C_l-K}{2}\right),\quad
\mathbb E[E_l^2]=\frac K{C_l},\quad
\sqrt{\mathbb E[E_l^2]}=\sqrt{\frac K{C_l}}.
\label{eq:app_projection_reference}
\end{equation}
Thus, $\tau_l=\sqrt{K/C_l}+\delta$ uses the root-mean-square energy
of this \emph{isotropic reference model} plus a heuristic margin.
It is not a distribution-free upper bound or a guarantee for actual
benign neural features. The derivation assumes $0<K<C_l$.

\subsection{Low-Rank Subspace Statistics}
\label{sec:appendix_low_rank}
SA-PCA retains 6--12 directions (mean 8.86) for the seven examined
Ariana Grande layers, each below $2\%$ of its channel dimension
(Table~\ref{tab:low_rank_stats}). These ranks establish compactness
of the sampled weighted features, not target specificity by themselves.

\begin{table}[h]
\centering
\caption{Layer-wise SA-PCA rank statistics for Ariana Grande ($>99\%$ retained weighted variance).}
\label{tab:low_rank_stats}
\resizebox{1.0\linewidth}{!}{%
\begin{tabular}{lcccc}
\toprule
\textbf{Cross-Attention Layer}
&
\textbf{Channel Dimension ($C_l$)}
&
\textbf{Retained Variance}
&
\textbf{Rank ($K$)}
&
\textbf{Capacity Ratio}
\\
\midrule
\texttt{mid\_block.attentions.0...attn2}
& 1280 & 0.9906 & 8  & 0.62\% \\
\texttt{up\_blocks.1.attentions.0...attn2}
& 1280 & 0.9903 & 10 & 0.78\% \\
\texttt{up\_blocks.1.attentions.1...attn2}
& 1280 & 0.9910 & 7  & 0.55\% \\
\texttt{up\_blocks.1.attentions.2...attn2}
& 1280 & 0.9933 & 6  & 0.47\% \\
\texttt{up\_blocks.2.attentions.0...attn2}
& 640  & 0.9911 & 12 & 1.88\% \\
\texttt{up\_blocks.2.attentions.1...attn2}
& 640  & 0.9911 & 10 & 1.56\% \\
\texttt{up\_blocks.2.attentions.2...attn2}
& 640  & 0.9911 & 9  & 1.41\% \\
\midrule
\textbf{Average / Total}
&
--
&
\textbf{$>0.99$}
&
\textbf{Avg: 8.86 (Total: 62)}
&
\textbf{Avg: 1.04\%}
\\
\bottomrule
\end{tabular}
}
\end{table}

\subsection{Evolution of Projection Energy}
\label{energy}
Fig.~\ref{fig:energy_fluctuation} shows the projection-energy trajectories
recorded during adapter training. The observed downward trend is
descriptive and does not by itself establish low false-positive gating
rates at inference.

\begin{figure}[h]
    \centering
    \includegraphics[width=1.0\linewidth]{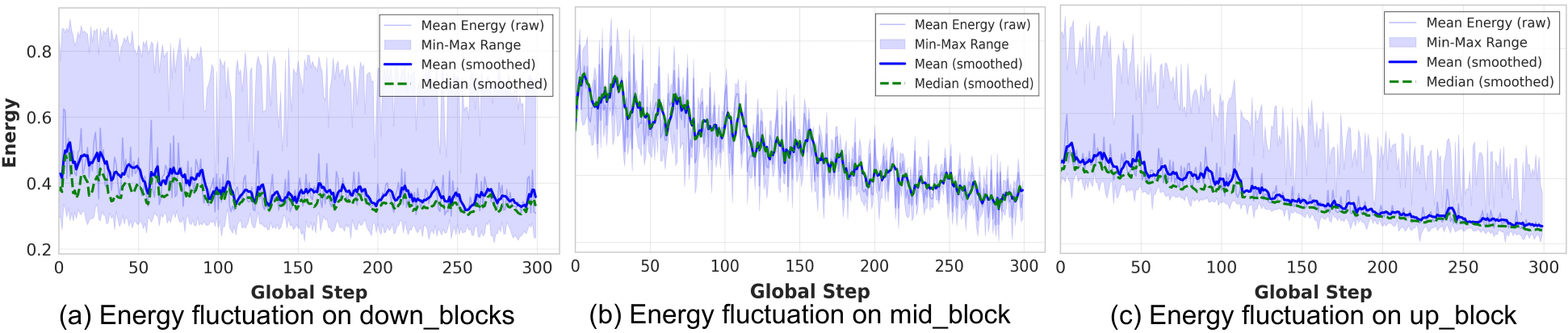}
    \caption{Projection energy in down, mid, and up U-Net blocks during adapter training. Lines show smoothed means; shading indicates the observed min--max range.}
    \label{fig:energy_fluctuation}
\end{figure}

\subsection{Hyperparameter Sensitivity}
\label{sec:appendix_hparams}
\paragraph{Regularization weight $\lambda$}
\label{sec:appendix_lambda}
Fig.~\ref{fig:lambda_ablation} reports the trade-off between target CS,
near-neighbor FID, and COCO FID. We use $\lambda=0.2$ for styles,
$0.5$ for IP characters, $0.6$ for nudity, and $0.7$ for identities.

\begin{figure}[!htbp]
\centering
\includegraphics[width=1.0\linewidth]{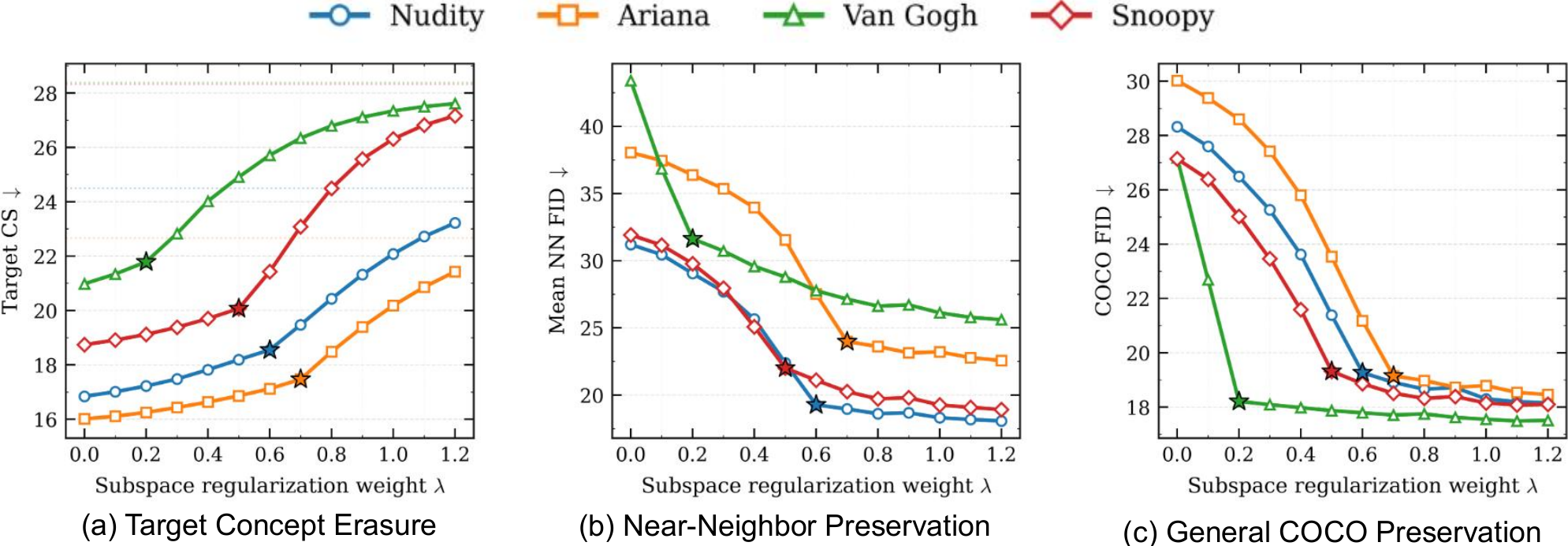}
\caption{Effect of regularization weight $\lambda$ on target CS, near-neighbor FID, and COCO FID. Stars indicate the selected settings.}
\label{fig:lambda_ablation}
\end{figure}

\paragraph{Spatial-gating sharpness $\alpha$}\label{sec:appendix_alpha}
We evaluate $\alpha\in\{20,40,60,80\}$ and use $\alpha=40$.
Smaller values smooth the sigmoid transition, whereas larger values
make it closer to a hard gate.

\paragraph{Temporal sharpness $\beta$}\label{sec:appendix_beta}
For $\eta(t)=\{1+\exp[\beta(t-t_{\mathrm{start}})]\}^{-1}$, the output
falls from $0.95$ to $0.05$ as the sigmoid input increases from
$-\ln19$ to $\ln19$. Hence the transition width in the units of $t$ is
\begin{equation}
\Delta t=\frac{2\ln19}{\beta}\approx\frac{5.89}{\beta}.
\label{eq:app_beta_width}
\end{equation}
With $\beta=1$ and $t$ indexing the 50-step inference schedule, the transition width is approximately six inference steps.

\section{Extended Evaluation}
\label{sec:app_qual}

\subsection{Artistic-Style Erasure}
\label{sec:app_style}
Table~\ref{tab:style_erasure} evaluates single- and joint-style erasure, examining both target suppression and the preservation of neighboring artistic styles and general image quality. GRACE achieves competitive erasure performance while substantially reducing unintended changes to non-target content, demonstrating its ability to localize semantic intervention and maintain a favorable erasure--preservation trade-off. 

\begin{table}[h]
    \centering
    \caption{Style erasure and preservation results. Erasure targets vary by task; target CS ($\downarrow$) measures suppression, while retained-concept CS ($\uparrow$) and FID ($\downarrow$) assess preservation. Target FID describes distributional shift, not a standalone erasure score.}
    \label{tab:style_erasure}
    \footnotesize
      \resizebox{0.85\linewidth}{!}{%
    \begin{tabular}{lcccccccccc}
        \toprule
        \multirow{2}{*}{Method} & \multicolumn{2}{c}{Van Gogh} & \multicolumn{2}{c}{Picasso} & \multicolumn{2}{c}{Monet} & \multicolumn{2}{c}{Rembrandt} & \multicolumn{2}{c}{General} \\
        \cmidrule(lr){2-3} \cmidrule(lr){4-5} \cmidrule(lr){6-7} \cmidrule(lr){8-9} \cmidrule(lr){10-11}
        & CS & FID & CS & FID & CS & FID & CS & FID & CS & FID \\
        \midrule
        \multicolumn{11}{l}{\textbf{\textit{Task 1: Erasing Van Gogh}}} \\
        \midrule
        SD v1.5 & 28.38 & \textcolor{gray}{-} & \textcolor{gray}{28.79} & - & \textcolor{gray}{28.52} & - & \textcolor{gray}{28.75} & - & 26.73 & - \\
        ESD & 21.88 & \textcolor{gray}{195.76} & \textcolor{gray}{26.16} & 94.88 & \textcolor{gray}{25.54} & 94.43 & \textcolor{gray}{26.72} & 90.54 & 26.43 & 21.96 \\
        Speed & 23.20 & \textcolor{gray}{166.40} & \textcolor{gray}{27.61} & 70.49 & \textcolor{gray}{27.93} & 35.03 & \textcolor{gray}{27.34} & 43.59 & 26.32 & 18.55 \\
        SPM & 22.26 & \textcolor{gray}{198.65} & \textcolor{gray}{27.98} & 55.39 & \textcolor{gray}{\textbf{28.48}} & 62.14 & \textcolor{gray}{28.51} & 43.89 & 26.35 & 19.22 \\
        Co-Erasing & \textbf{21.60} & \textcolor{gray}{192.34} & \textcolor{gray}{26.37} & 92.15 & \textcolor{gray}{25.79} & 81.86 & \textcolor{gray}{26.58} & 88.72 & 25.97 & 23.89 \\
        MACE & 22.63 & \textcolor{gray}{178.52} & \textcolor{gray}{26.91} & 85.64 & \textcolor{gray}{27.09} & 98.37 & \textcolor{gray}{27.47} & 79.25 & 25.12 & 27.53 \\
        Ours & 21.79 & \textcolor{gray}{\textbf{199.79}} & \textcolor{gray}{\textbf{28.42}} & \textbf{28.32} & \textcolor{gray}{28.37} & \textbf{30.71} & \textcolor{gray}{\textbf{28.69}} & \textbf{35.85} & \textbf{26.61} & \textbf{18.21} \\
        \midrule
        \multicolumn{11}{l}{\textbf{\textit{Task 2: Erasing Picasso}}} \\
        \midrule
        ESD & \textcolor{gray}{27.55} & 104.43 & 24.15 & \textcolor{gray}{170.59} & \textcolor{gray}{25.29} & 96.48 & \textcolor{gray}{28.00} & 91.24 & 25.64 & 22.62 \\
        Speed & \textcolor{gray}{28.27} & 39.79 & 24.37 & \textcolor{gray}{139.43} & \textcolor{gray}{27.86} & 36.32 & \textcolor{gray}{28.56} & 44.31 & 26.08 & 19.93 \\
        SPM & \textcolor{gray}{\textbf{28.36}} & 45.70 & 21.40 & \textcolor{gray}{\textbf{183.57}} & \textcolor{gray}{\textbf{28.50}} & 42.73 & \textcolor{gray}{\textbf{28.74}} & 41.43 & 26.17 & \textbf{19.24} \\
        Co-Erasing & \textcolor{gray}{27.30} & 91.87 & \textbf{20.61} & \textcolor{gray}{167.24} & \textcolor{gray}{24.91} & 93.91 & \textcolor{gray}{27.85} & 89.68 & 25.33 & 21.55 \\
        MACE & \textcolor{gray}{27.78} & 102.15 & 23.05 & \textcolor{gray}{158.63} & \textcolor{gray}{26.17} & 104.46 & \textcolor{gray}{28.27} & 76.72 & 25.67 & 24.87 \\
        Ours & \textcolor{gray}{27.29} & \textbf{35.12} & 21.25 & \textcolor{gray}{175.35} & \textcolor{gray}{28.46} & \textbf{32.95} & \textcolor{gray}{28.71} & \textbf{34.27} & \textbf{26.35} & 19.29 \\
        \midrule
        \multicolumn{11}{l}{\textbf{\textit{Task 3: Erasing Van Gogh \& Picasso}}} \\
        \midrule
        ESD & 20.45 & \textcolor{gray}{171.26} & 20.68 & \textcolor{gray}{145.27} & \textcolor{gray}{28.09} & 111.45 & \textcolor{gray}{25.79} & 114.73 & 24.84 & 28.12 \\
        Speed & \textbf{20.12} & \textcolor{gray}{180.42} & 21.42 & \textcolor{gray}{\textbf{177.32}} & \textcolor{gray}{28.32} & 48.37 & \textcolor{gray}{27.89} & 46.38 & 25.76 & 23.34 \\
        SPM & 21.39 & \textcolor{gray}{\textbf{201.37}} & 21.72 & \textcolor{gray}{173.46} & \textcolor{gray}{28.33} & 53.28 & \textcolor{gray}{28.63} & 57.04 & 26.02 & 23.27 \\
        Co-Erasing & 20.78 & \textcolor{gray}{195.68} & \textbf{19.74} & \textcolor{gray}{169.93} & \textcolor{gray}{25.61} & 105.17 & \textcolor{gray}{26.97} & 100.34 & 24.69 & 24.01 \\
        MACE & 21.86 & \textcolor{gray}{182.49} & 20.29 & \textcolor{gray}{161.75} & \textcolor{gray}{26.98} & 110.29 & \textcolor{gray}{27.71} & 105.36 & 23.64 & 27.92 \\
        Ours & 21.04 & \textcolor{gray}{182.57} & 21.05 & \textcolor{gray}{175.18} & \textcolor{gray}{\textbf{28.36}} & \textbf{33.41} & \textcolor{gray}{\textbf{28.67}} & \textbf{37.39} & \textbf{26.27} & \textbf{20.75} \\
        \bottomrule
    \end{tabular}
    }
\end{table}

\begin{figure}[h]
    \centering
    \includegraphics[width=0.90\linewidth]{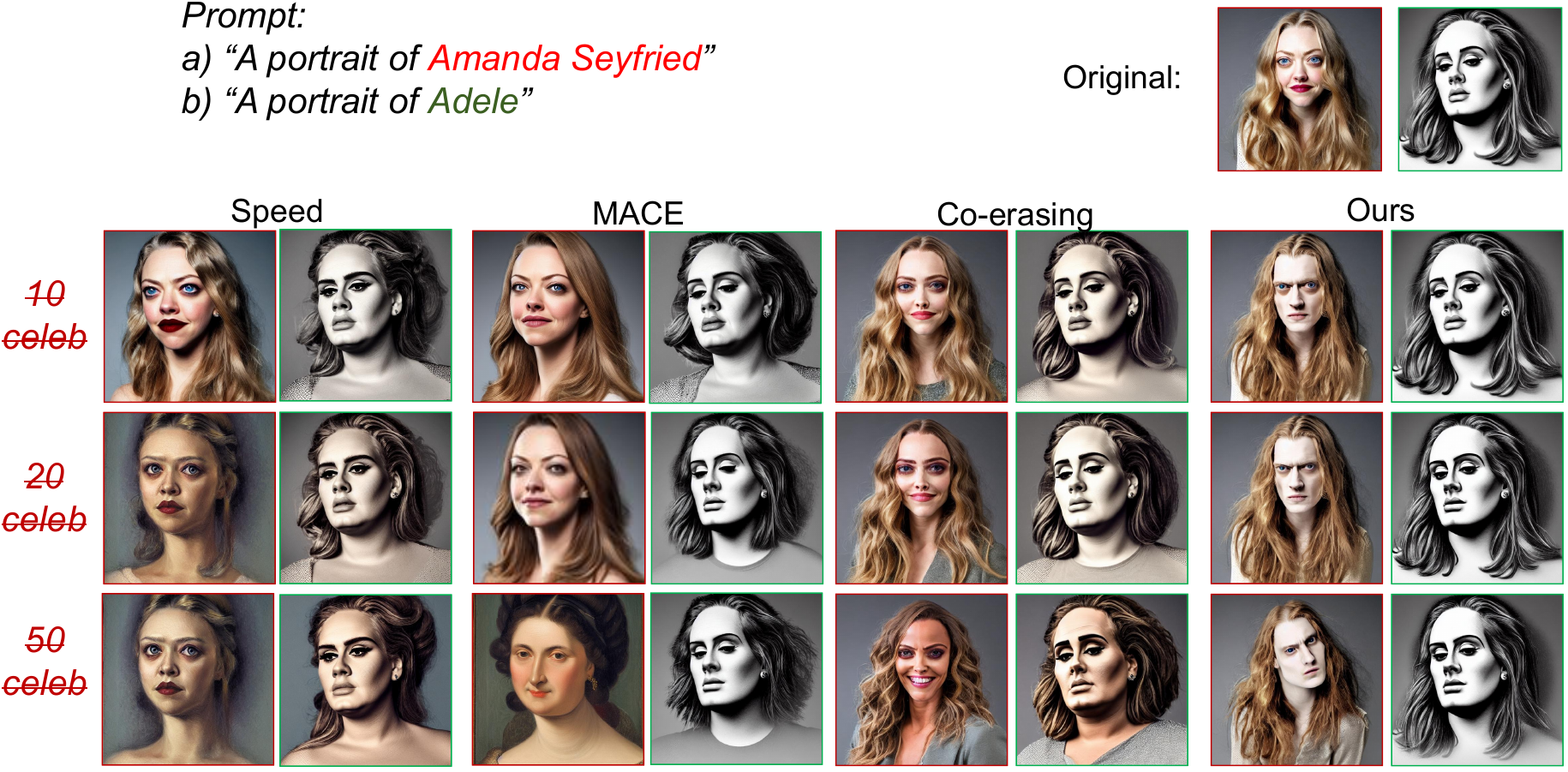}
    \caption{Identity erasure with 10, 20, and 50 targets. Red and green labels denote erased and retained identities.}
    \label{fig:20cele}
\end{figure}

\begin{figure}[h]
    \centering
    \includegraphics[width=0.9\linewidth]{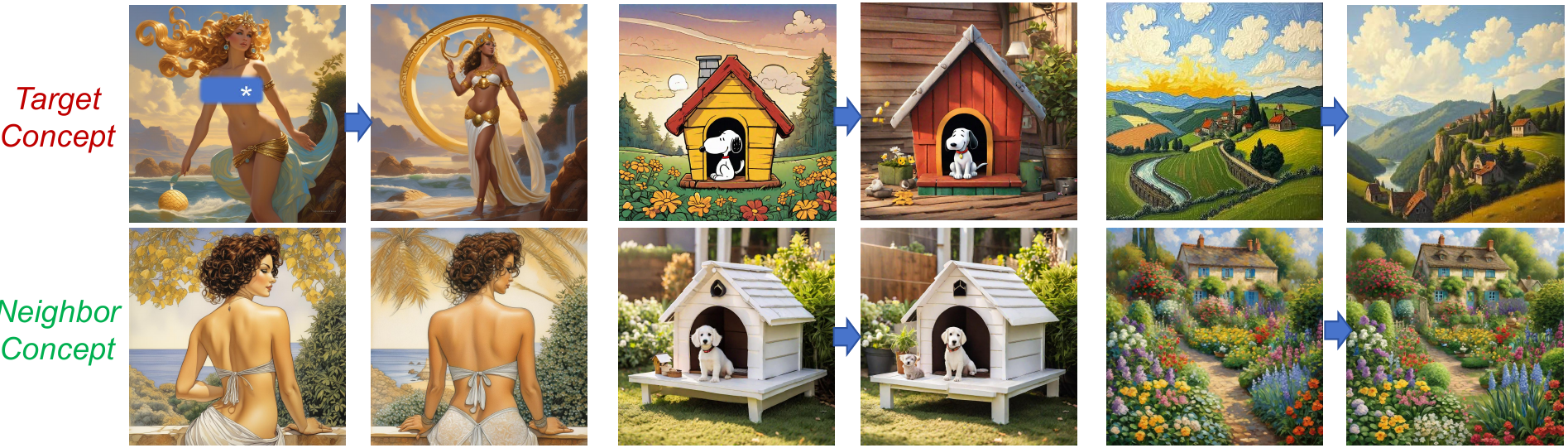}
    \caption{SDXL examples of nudity, Van Gogh style, and Snoopy erasure, together with their neighboring concepts.}
    \label{fig:sdxl}
\end{figure}

\subsection{Large-Scale Identity Erasure}
\label{sec:app_celebrities}
Fig.~\ref{fig:20cele} presents qualitative examples as the number of erased celebrity identities increases from 10 to 50. The results illustrate GRACE's ability to suppress multiple target identities while retaining the recognizable characteristics of a non-target identity, highlighting its potential for scalable and selective concept erasure.

\subsection{Cross-Backbone Examples}
\label{sec:app_cross_model}
Fig.~\ref{fig:sdxl} presents qualitative erasure results on SDXL v1.0, covering sensitive content, artistic styles, and IP characters alongside their neighboring concepts. The examples illustrate how GRACE suppresses target concepts while maintaining the visual coherence and semantic characteristics of non-target content.

\subsection{Adversarial Qualitative Evaluation}
\label{app:adversarial_qual}

We provide additional qualitative comparisons under
adversarial prompts to complement the quantitative
robustness.
Fig.~\ref{fig:app_aggressive} presents examples of
attacks targeting identity, nudity, and
IP-character concepts. These examples illustrate
the behavior of different erasure methods when
adversarial prompts attempt to recover the erased
concepts.

Fig.~\ref{fig:app_aggressive_style} further examines
artistic-style erasure under adversarial conditions,
using Van Gogh as the target concept and Monet as
a neighboring concept to be preserved. The
comparison highlights both target suppression and
the preservation of related non-target semantics.

\begin{figure}[t]
    \centering
    \includegraphics[width=\linewidth]{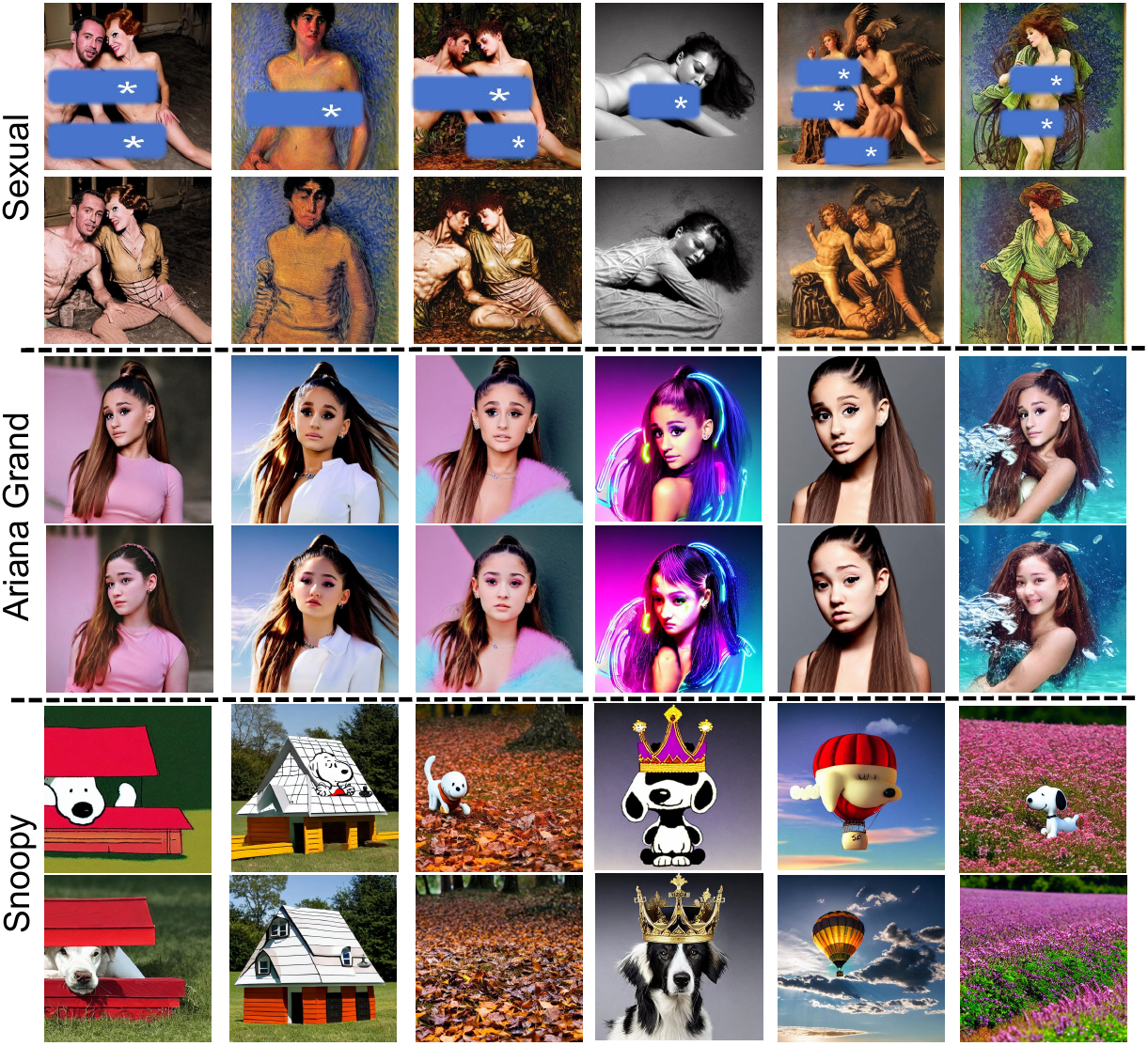}
    \caption{Qualitative comparison of concept erasure
    under adversarial prompts targeting
    Ariana Grande, nudity, and Snoopy.}
    \label{fig:app_aggressive}
\end{figure}

\begin{figure}[htbp]
    \centering
    \includegraphics[width=\linewidth]{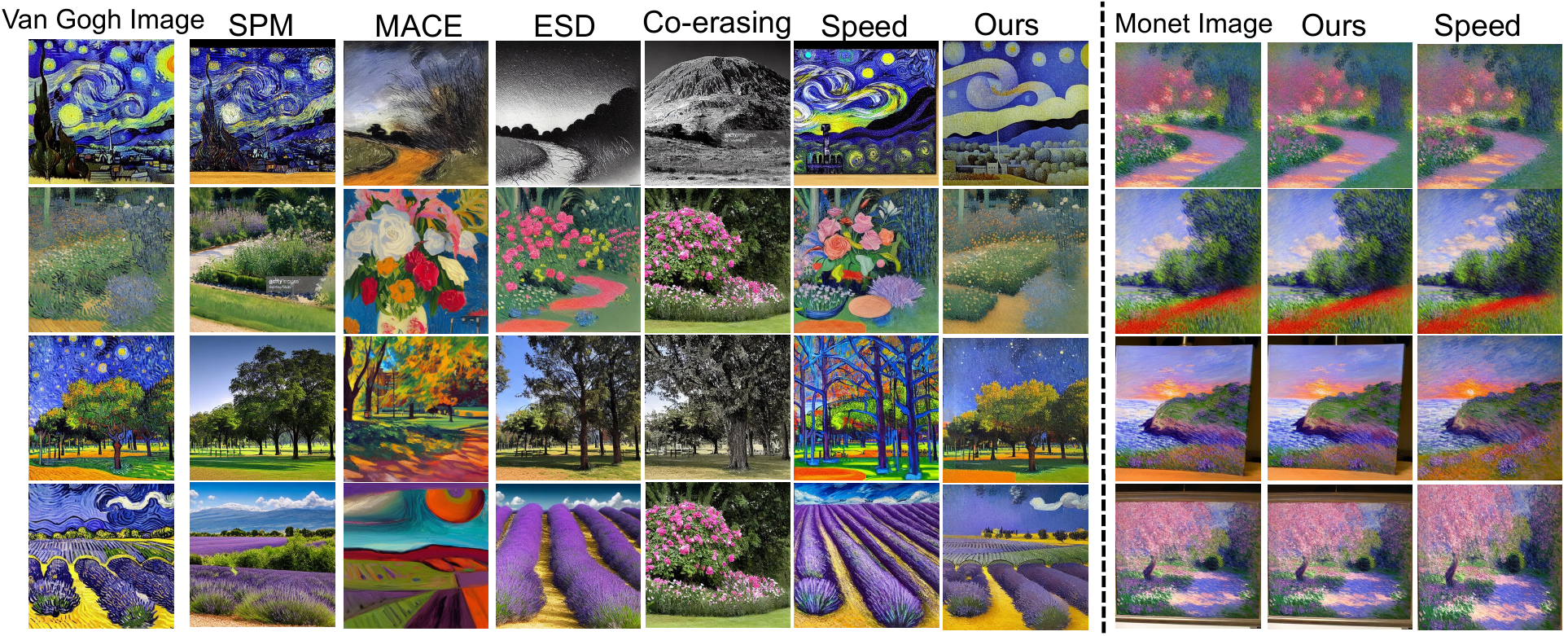}
    \caption{Qualitative comparison of Van Gogh style
    erasure under adversarial prompts and the
    preservation of the neighboring Monet style
    across different erasure methods.}
    \label{fig:app_aggressive_style}
\end{figure}

\end{document}